\documentclass[journal]{IEEEtran}
\usepackage{cite}
\usepackage{amsmath,amssymb,amsfonts,amsthm}
\usepackage{graphicx}
\graphicspath{ {./figures/} }
\usepackage{algorithmic}
\usepackage{algorithm}
\usepackage{hyperref}
\hypersetup{hidelinks}
\usepackage{xr}
\usepackage{textcomp}
\usepackage{subcaption}
\usepackage{ragged2e}
\usepackage{float}
\usepackage{caption}

\newcommand{\QED}{\ensuremath{\square}}
\renewenvironment{proof}[1][Proof.]
{\par\noindent{\itshape #1 }}
{\hspace*{\fill}~\QED\par}

\def\BibTeX{{\rm B\kern-.05em{\sc i\kern-.025em b}\kern-.08em
     T\kern-.1667em\lower.7ex\hbox{E}\kern-.125emX}}

\begin{document}
\newtheorem{theorem}{Theorem}
\newtheorem{remark}{Remark}
\newtheorem{lemma}{Lemma}
\newtheorem{conjecture}{Conjecture}

\title{Sparse Identification for Automatic Large-Scale Screening: A Constraint-Aware Framework with Ultra Fast Decoding Algorithm}
\author{Jianing Li, Li Chai,~\IEEEmembership{Senior Member,~IEEE}, Yingcheng Lai
\thanks{This work was supported by the National Natural Science Foundation of China under grant 62550085.}
\thanks{Jianing Li, Li Chai, and Yingcheng Lai are with the College of Control Science and Engineering, Zhejiang University, Hangzhou 310027, China (e-mail: lijn202409@zju.edu.cn; chaili@zju.edu.cn; laiyingcheng@zju.edu.cn). Corresponding author: Li Chai.}
}

\maketitle
\begin{abstract}
In the early stages of a pandemic, identification of a small number of infected individuals through large-scale screening is critical for pandemic control, yet remains challenging under limited reagents and testing capacity. Existing group testing methods suffer from either high computational complexity or low identification accuracy. Even worse, no available methods provide theoretically rigorous analysis for sparse identification with hard constraints caused by the sample usage constraint and the dilution effect existing ubiquitously in practical applications. In this article, we propose the Logic Screening method (LoSc), an ultra fast, accurate, and theoretically grounded framework for large-scale screening. LoSc introduces a novel decoding algorithm with a very simple selection strategy, achieving identification of all positives with only $\mathcal{O}\mathbf{\left(k\log n\right)}$ pooled tests. The decoding relies only on logical operations, enabling direct hardware implementation and yielding ultra fast computational implementation. Moreover, LoSc explicitly incorporates dilution and sample usage constraints into pooling designs, and establishes theoretical guarantees to guide optimal pooling configurations. Extensive simulations confirm the superior effectiveness, efficiency, and scalability. We believe LoSc offers a fast and reliable solution for automatic large-scale screening.
\end{abstract}

\begin{IEEEkeywords}
Large-scale screening, group testing, dilution effect, sample usage constraint, sparse identification
\end{IEEEkeywords}

\section{Introduction}
\label{sec:introduction}
Over the past two decades, outbreaks of respiratory viral infections have posed serious global health challenges \cite{zhou2020pneumonia}, leading to two large-scale pandemics, SARS and Middle East respiratory syndrome (MERS), and ultimately culminated in the COVID-19 pandemic, which has caused more than 700 million cases and over 7 million deaths worldwide \cite{Xinhui2020,WHO2023,Xiaoying2023}. To control the COVID-19 pandemic, governments and health agencies around the world used a variety of mitigation measures to curb viral transmission, including diagnostic testing, contact tracing and lockdowns. Among these, diagnostic testing on a massive scale played an important role in the identification of infected individuals. To date, over $1.5\times 10^{10}$ PCR tests have been performed worldwide \cite{hasell2020cross}. However, achieving sufficient testing coverage across populations remains a significant challenge \cite{cheng2020diagnostic}, largely due to disruptions in global supply chains for testing reagents and supplies, as well as limitations in testing capacity and financial resources \cite{Zichao2020,cleary2021using}. Although WHO has declared the end to COVID-19 as a public health emergency, the risk of new variants or entirely novel pathogens remains. WHO Director-General Dr.Tedros Adhanom Ghebreyesus emphasized the importance of comprehensive preparedness and response to address health emergencies at the 76th World Health Assembly \cite{world2023director}.

Past outbreaks and emerging threats highlight the urgent need for more efficient, scalable, and resilient diagnostic strategies to ensure both individual-level diagnosis and population-wide surveillance. Large-scale screening is very important in the early stages of a pandemic, when timely and accurate identification of infected individuals can prevent widespread transmission. However, when disease prevalence is extremely low, with only a few infections among thousands of individuals (i.e., $k\ll n$), identifying all positives with the fewest pooled tests becomes particularly challenging, especially under reagent shortages and limited testing capacity.

A wide range of pooling methods \cite{dorfman1943detection,westreich2008optimizing,smith2009use,bilder2020tests,da2022simulation,xia2022expectation} have been proposed to enhance diagnostic testing capabilities. Early studies, inspired by the classical Dorfman scheme \cite{dorfman1943detection}, developed hierarchical and array-based pooling designs. Westreich et al. \cite{westreich2008optimizing} evaluated two-stage minipools, three-stage hierarchical pools, and square arrays for acute HIV detection, reducing test usage by up to ten-fold at low prevalence. Similar designs were later applied to HIV treatment monitoring \cite{smith2009use}, using $10\times 10$ matrix and $5$-sample minipool methods on 155 samples, achieving a $50\%$ reduction in assays, but larger pools increased detection thresholds and turnaround time. During the COVID-19 pandemic, Dorfman testing implemented in a public health laboratory \cite{bilder2020tests} reduced test usage by $57\%$ with a pool size of 5 at $5\%$ prevalence, yet its efficiency degraded rapidly as prevalence increased. Multi-stage hierarchical designs could further improve detection power \cite{da2022simulation}, at the expense of increased logistical complexity and longer processing time. More recently, Xia et al. \cite{xia2022expectation} proposed XMAGT, an expectation maximization based adaptive pooling strategy that estimates prevalence during screening and dynamically selects pool sizes to maximize testing efficiency.

Subsequent research \cite{yi2020low, petersen2020practical, shental2020efficient,zismanov2024high,taufer2020rapid,ghosh2021compressed} focused on single-stage pooling methods based on compressed sensing (CS) and combinatorial group testing. CS-based methods exploit signal sparsity to estimate viral loads using only $\mathcal{O}(k\log n)$ measurements \cite{4472240}. The pooling process is modeled by a binary matrix $\Phi\in\{0,1\}^{m\times n}$, where $m$ denotes the number of pools. Early works \cite{yi2020low, petersen2020practical}
used measured viral loads of pools and convex optimization for recovery. Yi et al. \cite{yi2020low} adopted bipartite expander graphs as pooling matrices, and applied Basis Pursuit Denoising (BPDN) technique along with the brute-force search method for reconstruction, while Petersen et al. \cite{petersen2020practical} used Euler square matrices and applied the non-negative least absolute deviation regression (NNLAD) algorithm. Shental et al. \cite{shental2020efficient} proposed P-BEST, a non-adaptive pooling design based on the Reed-Solomon error correcting code,
which used only binary test outcomes and the CS-based decoding algorithm. P-BEST achieved an eight-fold reduction in tests at $1.3\%$ prevalence and was later validated on over 830,000 samples, requiring only 270,000 tests \cite{zismanov2024high}. Combinatorial group testing methods such as Combinatorial Orthogonal Matching Pursuit (COMP) and Definite Defectives (DD) \cite{chan2011non, chan2014non, aldridge2014group}, identify positives directly without estimating viral loads. COMP excludes definite negatives and declares the remaining as possible positives, while DD refines results by declaring a sample positive if it is the sole member of a positive pool. DD reduces false positives but often yields false negatives, posing a risk for pandemic control. Building on these methods, T$\ddot{\text{a}}$ufer \cite{taufer2020rapid} employed Shifted Transversal Designs with COMP decoding, and Ghosh et al. \cite{ghosh2021compressed} introduced Tapestry, using deterministic pooling matrices derived from Kirkman Triple Systems and a two-step decoder (COMP followed by NNOMP) to estimate viral loads, achieving over ten-fold efficiency gains at $1\%$ prevalence. 

Despite these advances, existing methods remain limited in reliability, scalability and practical applicability. Structured pooling matrices are available only for small dimensions (e.g., Kirkman Triple Systems for $m\leq 99$), making large-scale pooling designs dependent on heuristic or greedy search approaches. CS-based methods, despite their testing efficiency, incur substantial computational overhead, whereas combinatorial methods often compromise detection accuracy. On the other hand, dilution effects and limited sample volumes have significant impact on testing outcomes in practice, yet such hard constraints are difficult to handle in a stochastic setup. Together, these limitations highlight the necessity and importance for a fast, accurate, and scalable framework for large-scale screening.

In this work, we propose the Logic Screening method (LoSc), a fast and reliable framework for large-scale screening in the early stages of an outbreak. LoSc identifies all positive cases with the number of pooled tests comparable to CS-based methods, while maintaining low computational complexity and supporting hardware-efficient implementation. LoSc introduces a novel decoding algorithm that integrates logical elimination with candidate over-selection (Fig.~\ref{fig:constructure}), achieving a computational complexity of $\mathcal{O}\left(km\right)$. The algorithm first eliminates samples in negative pools, then iteratively applies a greedy selection strategy to select candidates that best explain the observed positive pools, until a predefined selection threshold is reached. Specifically, LoSc adopts Top-k and Top-2k sample selection strategies, which intentionally select more candidates than the expected number of infections,  thereby reducing both the number of pooled tests and false negatives.

Beyond the decoding algorithm, LoSc explicitly incorporates two critical constraints into the pooling designs: the dilution constraint and the sample usage constraint. These designs improve the scalability and practical applicability of LoSc under real-world biological and operational limitations. Consequently, three pooling designs are proposed: unconstrained, dilution-constrained, and sample-constrained designs, each tailored to specific application scenarios. The unconstrained pooling design is simple and flexible, allowing samples to be distributed randomly across pools, and is widely used in large-scale screening applications. The dilution-constrained pooling design strictly limits pool size to mitigate dilution effects, preserving assay sensitivity and ensuring reliable detection. The sample-constrained pooling design restricts the number of pools to which each sample can be assigned, thereby preventing sample exhaustion.

For each design, we derive a rigorous theoretical bound on the number of pooled tests $m$ required to identify all positives, accounting for the population size $n$, the number of infections $k$ and the corresponding pooling parameter ($p$, $P_r$ or $P_c$). We further propose conjectures on tighter bounds for the unconstrained and dilution-constrained pooling designs, and validate them through numerical simulations, indicating that they can capture the relationship between the number of pooled tests $m$ and the pooling parameters. Both theoretical analysis and numerical simulations show that LoSc identifies all positive samples with high probability using only $\mathcal{O}\left(k\log n\right)$ pooled tests, demonstrating its remarkable efficiency and reliability. Finally, we provide practical guidance for selecting pooling parameters and conservative estimates of the testing requirements to achieve near-optimal detection performance. 

Combining high testing efficiency, low computational complexity, and practical feasibility, LoSc delivers a rigorous and scalable framework for large-scale screening. Although LoSc is designed for epidemic surveillance, it can be generalized to other applications, such as screening for rare genetic mutations \cite{shental2010identification, nida2016highly} or identifying defective items in manufacturing.
\begin{figure*}[!ht] 
	\centering   
    \includegraphics[width=0.88\textwidth]{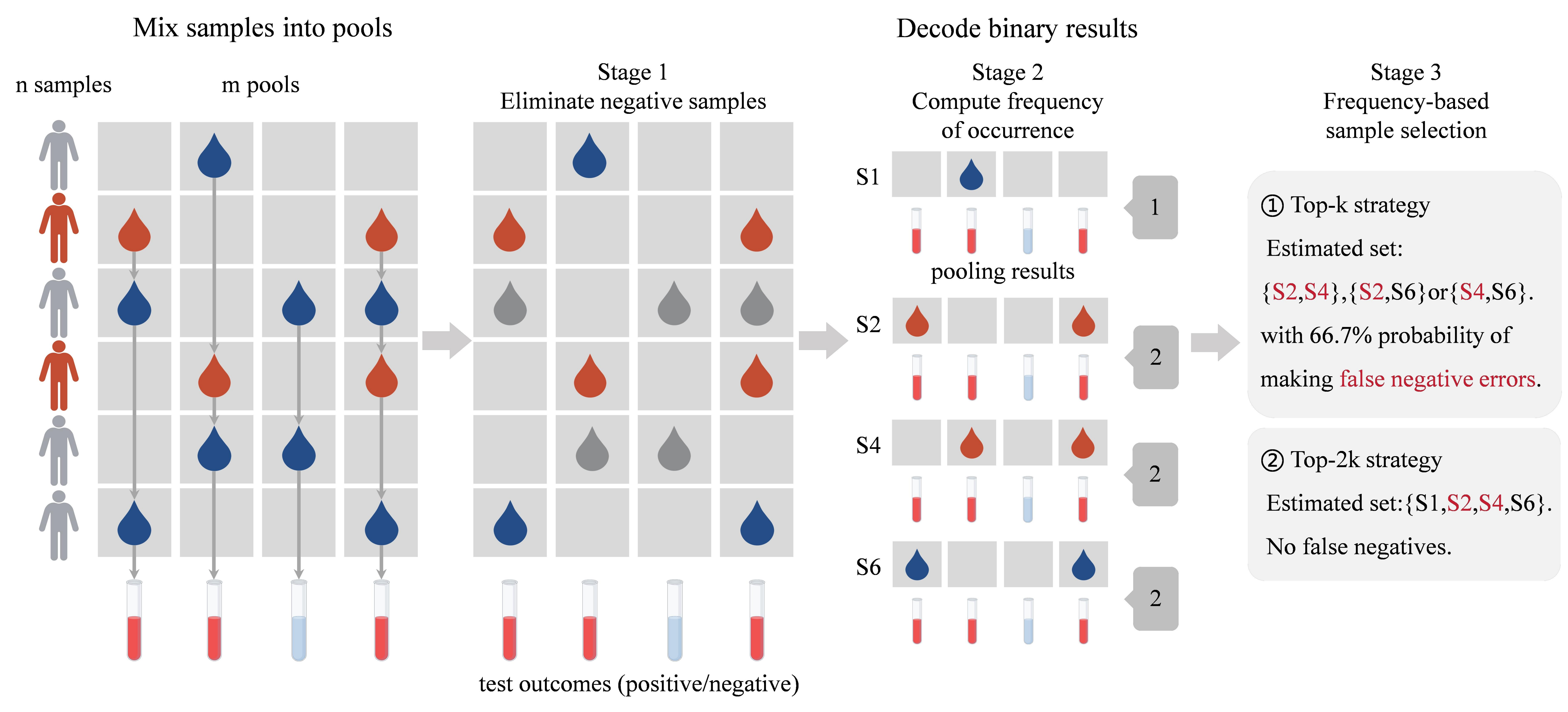}
    \captionsetup{position=below,justification=justified,singlelinecheck=false}
	\caption{\textbf{Overview of Logic Screening method.} Samples are assigned to pools according to a specific pooling design. In the first stage, samples appearing in negative pools are declared as negative. Then, compute the frequency of occurrence in positive pools for the remaining samples. Finally, select top $k$ or $2k$ samples as the estimated set.}
    \vspace{-0.6cm}
    \label{fig:constructure}
\end{figure*}

The remainder of this paper is organized as follows. We introduce the system model and problem formulation in Section~\ref{sec:problem statement}. Section~\ref{sec:method} presents the Logic Screening method and provides theoretical analysis. The simulation results are reported in Section~\ref{sec:results} and we conclude this paper in Section~\ref{sec:conclusion}.

\section{Problem Statement}\label{sec:problem statement}
Consider a large population of $n$ individuals. Denote the viral load vector as $x = {\left[ {{x_1}, \ldots ,{x_n}} \right]^T} \in {R^n}$, where $x_{i} \geq 0$ is the viral load of $i$-th sample ($x_{i}=0$ if $i$-th sample is uninfected). Let $S = \left\{ {i\left| {{x_i} > 0} \right.} \right\}$ be the set of positives. In the very early stages of a pandemic, the vector $x$ is extremely sparse, meaning that the number of infections $k=\left|S\right|$ is significantly smaller than the population size, i.e, $k \ll n$.

Let $m$ denote the number of pooled tests. The pooling process is represented by a binary matrix $\Phi \in \{0,1\}^{m \times n}$. Each row of $\Phi$ corresponds to a specific pool, with $\Phi_{i,j}=1$ indicating that the $j$-th sample is in the $i$-th pool, and 0 otherwise. To model the dilution caused by sample pooling, we introduce a diagonal scaling matrix $D\in R^{m\times m}$, where each diagonal entry is given by
\begin{align}
    {D_{i,i}} = \left\{ {\begin{array}{*{20}{l}}
  {{1}/{{\sum\limits_{j = 1}^n {{\Phi _{i,j}}} }},}&{{\text{if}}\;\sum\limits_{j = 1}^n {{\Phi _{i,j}}}  > 0,} \\ 
  {0,}&{{\text{otherwise.}}} 
\end{array}} \right.
\end{align}
Accordingly, the viral load measurement of the $i$-th pool ($i=1,\dots,m$) is
\begin{align}
    y_i = (D\Phi)_{i,:} x,
\end{align}
where ${\left( {D\Phi } \right)}_{i,:}$ denotes the $i$-th row of the matrix $D\Phi$. 

Since most diagnostic assays only detect viral material above the limit of detection (LOD) and report binary (positive/negative) results,  we define the binary measurement vector ${\tilde y \in \{0,1\}^{m}}$ as
\begin{align}\label{eq:binary_measurement}
    {\tilde y_i} = \left\{ {\begin{array}{*{20}{l}}
  {1,}&{{\text{if}}\,{y_i} \geq \text{LOD},} \\ 
  {0,}&{{\text{otherwise.}}} 
\end{array}} \right.
\end{align}
This implies that a pool is declared negative when the diluted viral load is below the LOD. Throughout, we assume that all measurements are noise-free.

Given the pooling matrix $\Phi$ and the test outcomes $\tilde{y}$, our goal is to develop decoding algorithms to identify positives (i.e., recover the set of positives $S$) with both low computational complexity and high accuracy. We aim to establish lower bounds on the number of tests $m$ required to identify all positives, and evaluate the performance of decoding algorithms under different pooling designs.

\section{Logic Screening Method: Identifying Positive Samples}\label{sec:method}
In this section, we introduce the Logic Screening method (Fig.~\ref{fig:constructure}), which uses the pooling matrix and the binary test results to identify positive samples. We prove rigorously that LoSc can achieve high detection accuracy and computational efficiency without the need for numerical optimization, making it suitable for large-scale screening. 

\subsection{Pooling Design}
In this work, we consider three pooling designs with distinct structural properties and practical implications. Each design can be automatically deployed on liquid-dispensing robots, making it an efficient and scalable screening tool.
\paragraph{Unconstrained Pooling Design} In this scenario, the pooling is constructed by independently assigning each sample to each pool with probability $p$. This corresponds to the pooling matrix $\Phi$ whose element $\Phi_{i,j}$ equals to 1 with probability $p$ and 0 with probability $1-p$. This pooling matrix ensures that samples are distributed randomly across pools. However, such randomness introduces practical challenges: some pools may become excessively large, increasing the risk of dilution and compromising assay reliability, while some samples may be assigned to many pools, risking sample depletion. Therefore, we propose the following two constrained pooling designs.
\paragraph{Dilution-Constrained Pooling Design} As already mentioned, one key challenge in pooling is the dilution effect: if pools are too large, positive samples may be diluted below the LOD, yielding false negatives. To avoid the dilution effect, we impose a hard constraint on the pool size, requiring each pool to contain exactly $P_{r}$ samples, that is, each row of $\Phi$ has exactly $P_{r}$ nonzero entries (Fig.~\ref{fig:constrained-designs}(a)). To this end, for each row $i$ of $\Phi$, we select uniformly $P_r$ indices from the set $\{1,\dots,n\}$ and set $\Phi_{i,j}=1$ for each selected index $j$. Compared with the unconstrained design, this design mitigates the risk of dilution and improves detection reliability.
\paragraph{Sample-Constrained Pooling Design}
This design considers the limitation of finite availability of samples, which is another key challenge in pooling. In this design, we put a hard constraint on the column of the pooling matrix. In particular, each column of $\Phi$ represents a sample and contains at most $P_{c}$ nonzero entries, ensuring that each sample is assigned to at most $P_{c}$ pools (Fig.~\ref{fig:constrained-designs}(b)). For each column $j$ of $\Phi$, we perform $P_c$ independent selections from the set $\{1,\dots,m\}$ and set $\Phi_{i,j}=1$ if index $i$ is selected at least once. This design is particularly advantageous in practical scenarios where sample volumes are limited or repeated pipetting must be minimized, ensuring efficient use of scarce biological resources. 
\begin{figure}[!ht]
    \centering   
    \captionsetup{position=below,justification=centering,singlelinecheck=false}
    \subcaptionbox{Dilution-constrained design}[1\linewidth]{
        \includegraphics[width=1\linewidth]{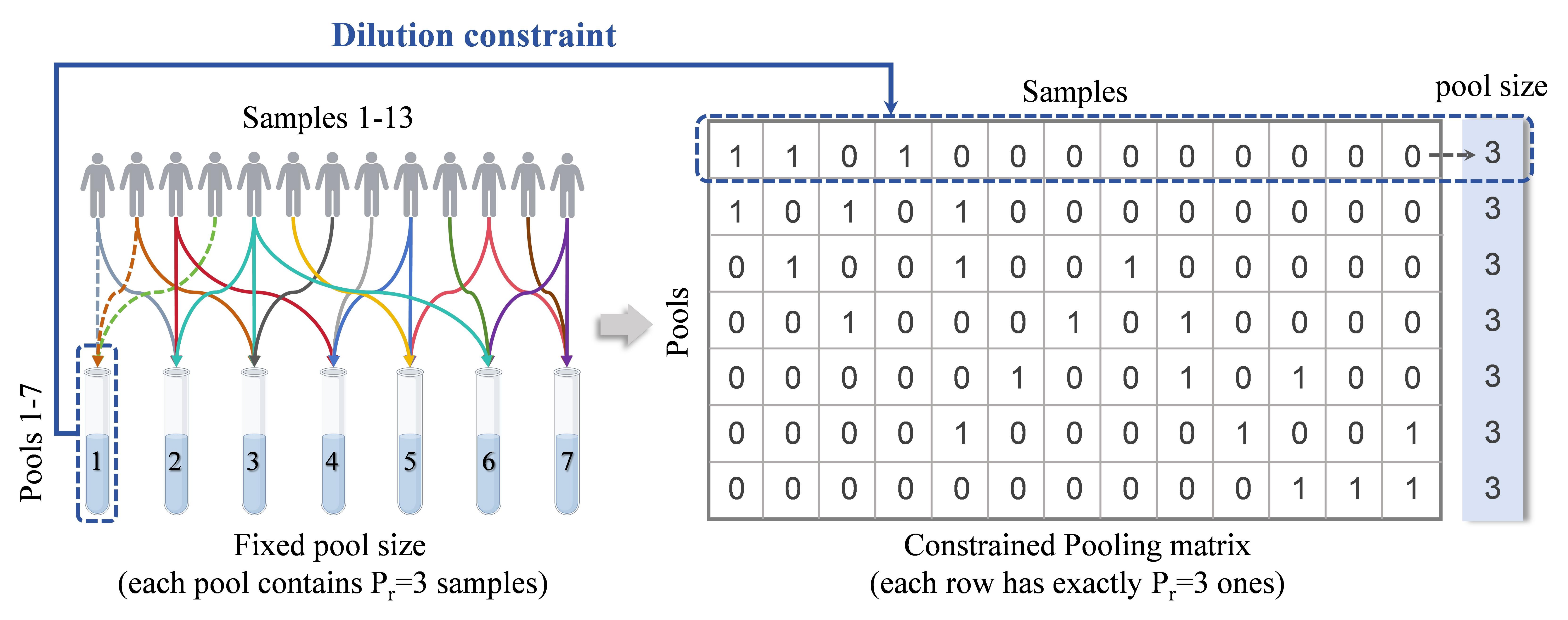}
        }
    \subcaptionbox{Sample-constrained design}[1\linewidth]{
        \includegraphics[width=1\linewidth]{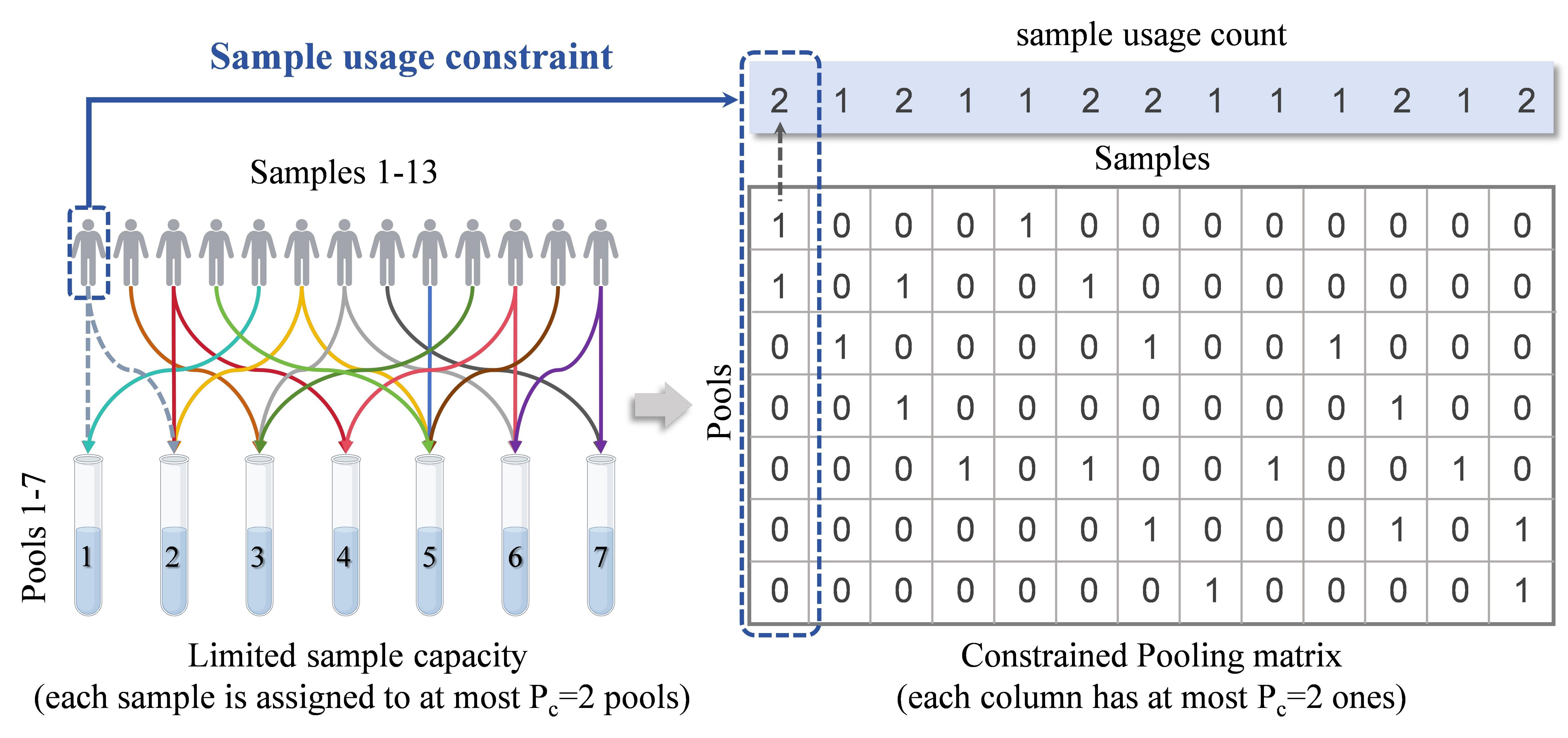}
        }
    \captionsetup{position=below,justification=justified,singlelinecheck=false}
    \caption{\textbf{Overview of pooling designs with different constraints.} (a) Dilution-constrained design. The pool size is strictly limited to mitigate dilution effects. The corresponding pooling matrix $\Phi$ has exactly $P_r$ nonzero entries per row. (b) Sample-constrained design. Sample usage is restricted to prevent exhaustion. The corresponding pooling matrix $\Phi$ has at most $P_c$ nonzero entries per column.}
    \vspace{-0.6cm}
    \label{fig:constrained-designs}
\end{figure}

\subsection{Decoding Algorithm} 
The proposed decoding algorithm can quickly recover the set of positives $S$ from the binary test results ${\tilde y}$, with a computational complexity of $\mathcal{O}\left(km\right)$. The algorithm consists of three stages. First, it eliminates samples that appear in any negative pool and obtains a candidate set $C \subseteq \left\{ {1, \ldots ,n} \right\}$. Next, for each sample in the candidate set $C$, we count its frequency of occurrence in positive pools and select the top $k$ samples as the estimated set $\hat{S}_{k}$. For higher detection accuracy, we select the top $2k$ samples as $\hat{S}_{2k}$. The detailed procedure is given in Algorithm~\ref{decode_combinatorial}.
\begin{algorithm}[!ht]
	\caption{Decoding algorithm in LoSc}\label{decode_combinatorial} 
	\begin{algorithmic}[1]
            \REQUIRE Binary measurement vector ${\tilde y}\in\{0,1\}^{m}$, pooling matrix $\Phi\in\{0,1\}^{m\times n}$.
            \ENSURE Estimated sets ${{\hat S}_k}$ and ${{\hat S}_{2k}}$.

            \STATE Candidate set $C \leftarrow \left\{ {1, \ldots ,n} \right\}$, occurrence frequency list $B=[B_{1},\dots,B_{n}]$;
            \FOR{$i=1,\dots,m$}
                \IF{${\tilde y}_{i}=0$}
                    \STATE $C \leftarrow C\backslash \left\{ {l\left| {{\Phi _{i,l}} = 1} \right.} \right\}$;
                \ELSE
                    \FOR{$j\in \left\{ {l\left| {{\Phi _{i,l}} = 1} \right.} \right\}$}
                        \STATE ${B_j}\leftarrow {B_j} + 1$;
                    \ENDFOR
                \ENDIF
            \ENDFOR
            \STATE Sort the candidate set $C$ by $B_{j}$ (frequency of occurrence in positive pools) in descending order to obtain an ordered sequence $\left( {{j_1}, \ldots ,{j_{\left| C \right|}}} \right)$ such that ${B_{{j_1}}} \geq \cdots  \geq {B_{{j_{\left| C \right|}}}}$;
            \STATE ${{\hat S}_k} \leftarrow \left\{ {{j_1}, \ldots ,{j_k}} \right\}$;
            \IF{$\left| C \right| > 2k$}
                \STATE ${{\hat S}_{2k}} \leftarrow \left\{ {{j_1}, \ldots ,{j_{2k}}} \right\}$;
            \ELSE
                \STATE ${{\hat S}_{2k}} \leftarrow C$;
            \ENDIF
            \\\RETURN ${\hat S}_{k}$ and ${\hat S}_{2k}$.
	\end{algorithmic}
\end{algorithm}

The process of decoding is equivalent to computing the inner product between the binary measurement vector $\tilde y$ and each column of the pooling matrix $\Phi$ corresponding to a candidate sample. The number of positive pools is $\left| \mathrm{supp}(\tilde y)\right|=\mathcal{O}\bigl(m\bigr)$ and the size of the candidate set is $\lvert C\rvert=\mathcal{O}\bigl(k\bigr)$. For each sample $j\in C$, $B_{j}$ is the number of positive pools containing sample $j$. Clearly, $B_{j}\le \left| \mathrm{supp}(\tilde y)\right|$. Hence, the computational complexity of the decoding algorithm is 
\begin{align}
    \sum_{j\in C} B_j 
      \le\lvert C\rvert\lvert \mathrm{supp}(\tilde y)\rvert=\mathcal{O}\left(km\right).
\end{align}

\begin{figure*}[!t]
    \centering
    \includegraphics[width=0.88\textwidth]{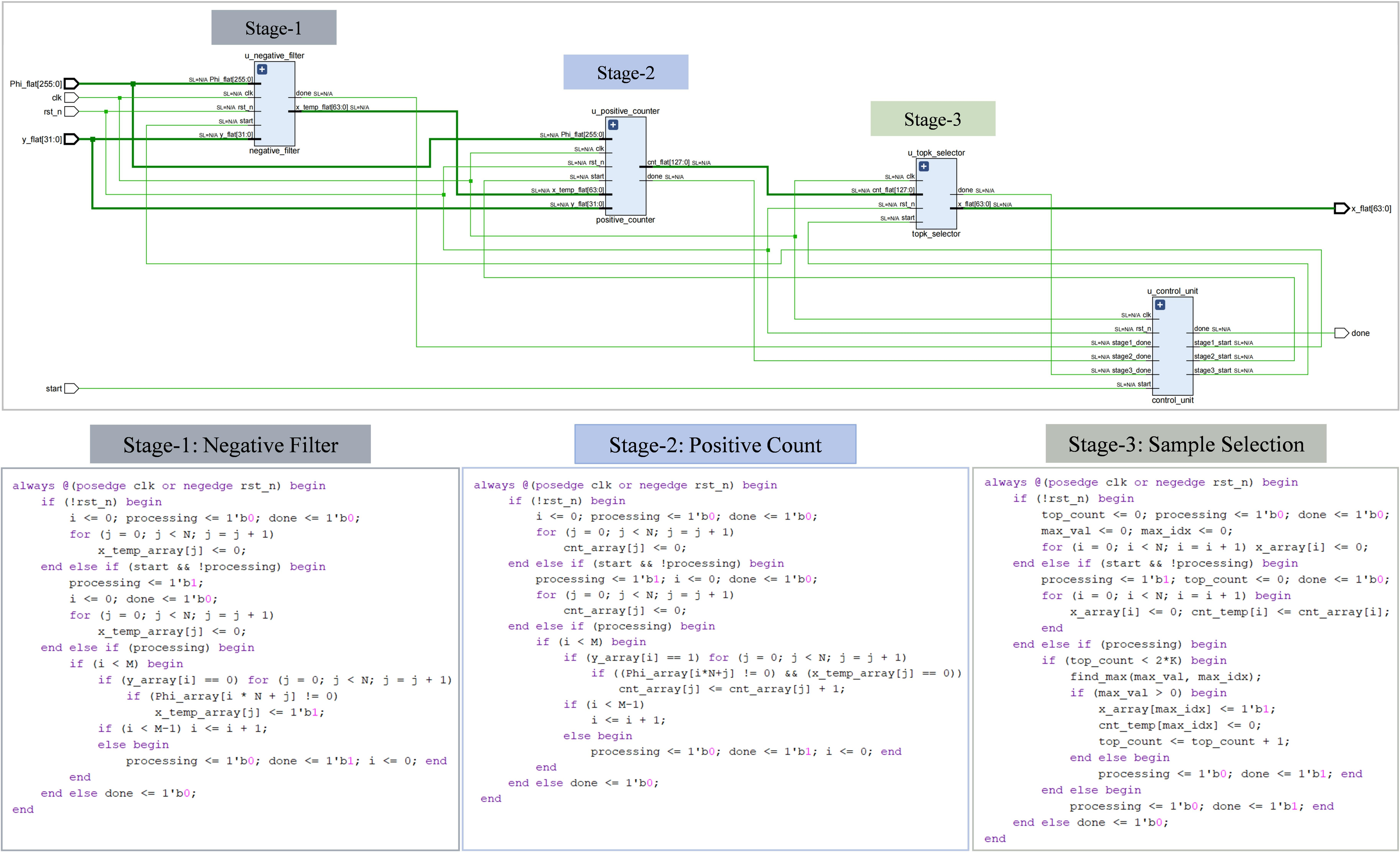}
    \caption{\textbf{Schematic diagram and Verilog code of the decoding algorithm.}}
    \vspace{-0.6cm}
    \label{fig:decoding_algorithm}
\end{figure*}
More importantly, the decoding algorithm can be implemented using logical operations, making it highly suitable for quick deployment and facilitating integration into existing workflows and automation systems. We present the schematic diagram and Verilog code of the decoding algorithm in Fig.~\ref{fig:decoding_algorithm}.

The sample selection strategy is motivated by the observation that true positive samples are more likely to appear in positive pools than negative ones across all pooling designs. For illustration, in the unconstrained design, each positive sample is independently assigned to a pool with probability $p$, whereas a negative sample appears in a positive pool with probability $p(1-(1-p)^k)$, which is much smaller than $p$ in the sparse pooling regime. By ranking samples according to their frequency of occurrence in positive pools, the algorithm prioritizes those with a higher probability of being truly positive, thus improving the detection accuracy.

Selecting the top $k$ samples generally suffices, but risks missing true positives. To address this issue, the algorithm adopts an over-selection strategy that chooses the top $2k$ samples as estimated set, which significantly reduces false negative errors and the number of pooled tests compared to the Top-k strategy, at the expense of additional follow-up tests.

\subsection{Performance Guarantees}
In this subsection, we establish mathematical guarantees for our decoding algorithm under three pooling designs, in which the number of samples and pooled tests are denoted as $n$ and $m$ respectively. For each design, we shall derive the sufficient condition on the number of pooled tests $m$ required to ensure the success rate of identification, which is measured by the probability of events $S \subseteq {{\hat S}_k}$ and $S \subseteq {{\hat S}_{2k}}$. The following theorems, proved in the appendices, show that the success rates can be very close to 1 under mild conditions. 

{\begin{theorem}[Unconstrained Pooling Design]\label{thm:unconstrained}
    In the unconstrained pooling design, let $p$ denote the probability of $\Phi_{i,j}=1$. Let $\hat{S}_k$ and $\hat{S}_{2k}$ be the estimated sets of Algorithm~\ref{decode_combinatorial}. Then, we have $\mathbb{P}\{S \subseteq {{\hat S}_k}\}\geq 1-\epsilon$ for arbitrarily small $\epsilon>0$, provided that the number of pooled tests satisfies
    \begin{align}
        m \geq \frac{{2\left(1-p(1-p)^k\right)^2\log \left( {2n/\epsilon}\right)}}{{{p^2}{{\left( {1 - p} \right)}^{2k + 2}}}}.
    \end{align}
    Moreover, we have $\mathbb{P}\{S \subseteq {{\hat S}_{2k}}\}\geq 1-\delta$, where
    \begin{align}
        \delta  = {\left( {{{e\epsilon }}/{{{(k + 1)}}}} \right)^{k + 1}} + {k\epsilon }/{{n}}.
    \end{align}
\end{theorem}}

The proof of Theorem~\ref{thm:unconstrained} is given in Appendix~\ref{pf:unconstrained}.

\textit{Optimal $p$ for Unconstrained Pooling Design:} For fixed $n$, $k$ and $\epsilon$, the number of pooled tests $m$ scales inversely with the term $\frac{p^2(1-p)^{2k+2}}{\left(1-p(1-p^k)\right)^2}$. To reach the minimum of $m$, we can obtain the approximate value of optimal $p$ as follows
\begin{align}
    p^*\approx 1/(k+1).
\end{align}
This choice balances the sparsity and coverage of the pooling matrix, providing a highly cost-efficient design.

{\begin{theorem}[Dilution-Constrained Pooling Design]\label{thm:dilution_constrained}
    In the dilution-constrained pooling design, let $P_r$ denote the pool size. Let $\hat{S}_k$ and $\hat{S}_{2k}$ be the estimated sets of Algorithm~\ref{decode_combinatorial}. Then, we have $\mathbb{P}\{S \subseteq {{\hat S}_k}\}\geq 1-\epsilon$ for arbitrarily small $\epsilon>0$, provided that the number of pooled tests satisfies
    \begin{align}
        m \geq \frac{{2\log \left( {2n/\epsilon} \right)}}{{{{\left( {{P_r}/n} \right)}^2}{{\left( {1 - {P_r}/n} \right)}^2}{{\left( {\binom{n-k-1}{P_r-1}/\binom{n-1}{P_r-1}} \right)}^2}}}.
    \end{align}
    Moreover, we have $\mathbb{P}\{S \subseteq {{\hat S}_{2k}}\}\geq 1-\delta$, where
    \begin{align}
        \delta  = {\left( {{{e\epsilon }}/{{{(k + 1)}}}} \right)^{k + 1}} + {k\epsilon }/{{n}}.
    \end{align}
\end{theorem}}

We provide the proof in Appendix~\ref{pf:dilution_constrained}.

\textit{Optimal $P_r$ for Dilution-Constrained Pooling Design:} We analyze the denominator and optimize $P_r$ to minimize $m$ for fixed $n$, $k$ and $\epsilon$. Note that $P_r$ must be an integer and $P_r\leq n$, the approximate value of optimal $P_r$ is given by
\begin{align}
    {P_{r}^*} \approx \left\lceil {n/k} \right\rceil.
\end{align}

{\begin{theorem}[Sample-Constrained Pooling Design]\label{thm:sample_constrained}
    In the sample-constrained pooling design, let $P_c$ denote the maximum number of pools that each sample can participate in. Let $\hat{S}_k$ and $\hat{S}_{2k}$ be the estimated sets of Algorithm~\ref{decode_combinatorial}. Then, we have $\mathbb{P}\{S \subseteq {{\hat S}_k}\}\geq 1-\epsilon$ for arbitrarily small $\epsilon>0$, provided that the number of pooled tests satisfies
    \begin{align}
        m \geq kP_c\left({(n-k)}/{\epsilon}\right)^{\frac{1}{P_c}}.
    \end{align}
    Moreover, we have $\mathbb{P}\{S \subseteq {{\hat S}_{2k}}\}\geq 1-\delta$, where
    \begin{align}
        \delta  = {\left( {{{e\epsilon }}/{{(k + 1)}}} \right)^{k + 1}}.
    \end{align}
\end{theorem}}

The proof of Theorem~\ref{thm:sample_constrained} is provided in Appendix~\ref{pf:sample_constrained}.

\textit{Optimal $P_c$ for Sample-Constrained Pooling Design:} To reach the minimum of $m$, we can obtain the approximate value of optimal $P_c$ as follows
\begin{align}
    {P_{c}^*} \approx \left\lceil \log\left((n-k)/\epsilon\right) \right\rceil.
\end{align}

\begin{remark}\label{theorem_remark}
To derive the bounds on the number of pooled tests $m$ in Theorems \ref{thm:unconstrained}-\ref{thm:sample_constrained}, we apply union bounds and several relaxations, which tend to overestimate the actual number of pooled tests required for a given success rate. In fact, numerical experiments show that exact recovery can be achieved with fewer pooled tests than those predicted by the theorems.

Additionally, under the same number of pooled tests, the Top-2k strategy achieves a higher success rate than the Top-k strategy. Specifically, for the unconstrained and dilution-constrained pooling designs, the ratio ${\delta}/{\epsilon}$ satisfies
\begin{align}
    {\delta}/{\epsilon}  = \epsilon^k {\left( {{{e}}/{{{(k + 1)}}}} \right)^{k + 1}} + {k}/{n} \approx 0,
\end{align}
since $k\ll n$. For the sample-constrained pooling design, the ratio is given by
\begin{align}
    {\delta}/{\epsilon}=\epsilon^k \left({e}/{(k+1)}\right)^{k+1} \approx 0.
\end{align}
In both cases, the ratio ${\delta}/{\epsilon}$ decays rapidly as $k$ increases, indicating that $\delta\approx 0$ and the success rate of Top-2k strategy is almost 1. Thus, expanding the estimated set substantially improves success rate at the cost of $k$ additional tests.
\end{remark}

\subsection{Conjectures on Tighter Bounds}
As noted in Remark~\ref{theorem_remark}, the rigorously theoretical bounds on $m$ in Theorems \ref{thm:unconstrained}-\ref{thm:sample_constrained} are generally conservative. There might be tighter bounds for the unconstrained and dilution-constrained pooling designs. Let $\epsilon\in\left(0,1\right)$ denote the target failure probability of the Top-k strategy. For sufficiently large $n$, the number of pooled tests $m$ scales as follows.
\begin{conjecture}\label{unconstrained_conjecture} 
    In the unconstrained pooling design, let $p$ denote the probability of $\Phi_{i,j}=1$. Let $\hat{S}_k$ be the estimated set of Algorithm~\ref{decode_combinatorial}. We have $\mathbb{P}\{S \subseteq {{\hat S}_k}\}\geq 1-\epsilon$ for arbitrarily small $\epsilon>0$, provided that
    \begin{align}
        m = \mathcal{O}\left(\frac{{\log \left( {{\left( {n - k} \right)}/\epsilon } \right)}}{{-\log\left(1 - {p}{{\left({1 - p} \right)}^{k}}\right)}}\right).
    \end{align}
\end{conjecture}
\begin{conjecture}\label{dilution_constrained_conjecture}
    In the dilution-constrained pooling design, let $P_r$ denote the pool size. Let $\hat{S}_k$ be the estimated set of Algorithm~\ref{decode_combinatorial}. We have $\mathbb{P}\{S \subseteq {{\hat S}_k}\}\geq 1-\epsilon$ for arbitrarily small $\epsilon>0$, provided that
    \begin{align}
        m = \mathcal{O}\left( {\frac{{\log \left( {{\left( {n - k} \right)}/\epsilon } \right)}}{{ - \log \left( {1 - \left(P_{r}/n\right) \cdot \left(\binom{n-k-1}{P_{r}-1}/\binom{n-1}{P_r-1}\right)} \right)}}} \right).
    \end{align}
\end{conjecture}

The above conjectures are supported by our experimental results, whereas rigorous proofs are left for future research.

\section{Results}\label{sec:results}
In this section, we comprehensively evaluate the proposed method. Experiments are conducted under various configurations of population size $n\in\{5000, 10000\}$, number of infections $k$, where ${k}/{n}$ ranges from 0.02\% to 10\%. The parameters of pooling designs are given as follows. In the unconstrained design, the pooling parameter $p$ takes values in $\{2^{-10}, 2^{-9}, \ldots ,2^{-1}\}$. In the dilution-constrained design, $P_r$ is chosen on a logarithmic scale from $\frac{n}{2^{10}}$ to $\frac{n}{2^1}$. In the sample-constrained design, the maximum number of pools per sample $P_c$ ranges from 2 to 55. For each parameter combination $\left( {n,k} \right)$ and the corresponding pooling parameter ($p$, $P_r$ or $P_c$), a series of $m$ (ranging from 1 to $n$) are simulated over 200 trials. 

In each trial, the pooling matrix $\Phi$ is generated randomly according to the specified pooling design. For the unconstrained design, each element $\Phi_{i,j}$ equals to 1 with probability $p$. For the dilution-constrained design, each row of $\Phi$ has exactly $P_{r}$ ones. For the sample-constrained design, each column of $\Phi$ has at most $P_{c}$ ones. Considering that viral load varies by many orders of magnitudes across individuals \cite{cleary2021using}, the vector $x$ is generated with nonzero entries drawn uniformly from $[10^2, 10^6]$ following \cite{ghosh2021compressed, bharadwaja2022recovery}, where ${\left\| x \right\|_0} = k$. All samples are assumed to be pooled in equal volumes and the binary measurement vector $\tilde y$ is generated according to Eq.(\ref{eq:binary_measurement}), where LOD is set to 1 virus copy/ml. A trial is successful if all positives are identified. The recovery success rate is the fraction of successful trials. 

\subsection{Near-Perfect Recovery with Minimal Pooled Tests}
To demonstrate the effectiveness of our method, we evaluate the performance under three pooling designs. We consider two large-scale scenarios: $n=5000$, $k=10$ and $n=10000$, $k=20$. Fig.~\ref{fig:combinatorial_success_rate_optimal_p} shows the recovery success rate as a function of the number of pooled tests $m$, and Table~\ref{tab:combinatorial_recovery_performance} summarizes the recovery performance under different pooling designs and sample selection strategies. The results demonstrate that LoSc achieves exact recovery with remarkably few tests across all pooling designs. For $n=5000$ and $k=10$, the Top-2k strategy attains over $99.5\%$ recovery with only $m=216$ pooled tests (less than $4.4\%$ of the population). Similarly, for $n=10000$ and $k=20$, near-perfect recovery is achieved with $m=420$ pooled tests. Among the three pooling designs, the sample-constrained design slightly outperforms the others, highlighting its structural advantage, which ensures each sample is assigned to at least one pool. These results confirm the ability of LoSc to identify all infected individuals with far fewer tests than the population size.
\begin{figure}[!t]
    \centering
    \includegraphics[width=1\linewidth]{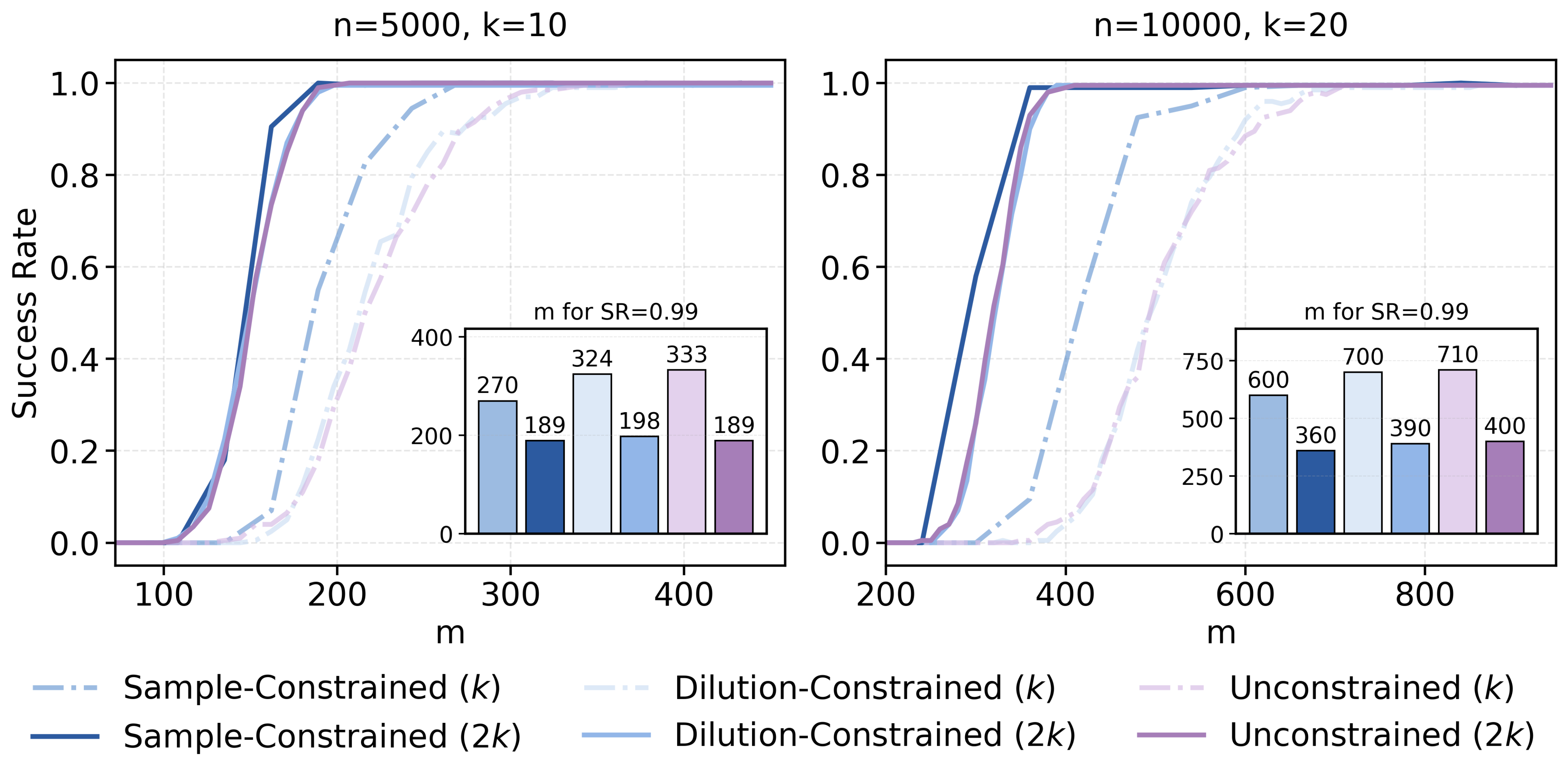}
    \caption{\textbf{ Recovery performance of LoSc under optimal pooling designs.} Recovery success rate as a function of $m$ under three pooling designs, each evaluated at its optimal setup. For $n=5000$ and $k=10$, the optimal pooling parameters are $p=0.0625$, $P_r=312$ and $P_c=19$. For $n=10000$, $k=20$, the optimal pooling parameters are $p=0.0312$, $P_r=312$ and $P_c=19$. Insets present minimum number of pooled tests required for a $99\%$ success rate.}
    \vspace{-0.6cm}
    \label{fig:combinatorial_success_rate_optimal_p}
\end{figure}

\begin{table*}[!t]
\caption{Recovery success rate of LoSc under optimal pooling designs}\label{tab:combinatorial_recovery_performance}
\resizebox{\textwidth}{!}{
\begin{tabular}{cccccccccccccc}
\hline
& \multicolumn{6}{c}{\boldmath{$n=5000$, $k=10$}}    &  & \multicolumn{6}{c}{\boldmath{$n=10000$, $k=20$}}                                             \\ \cline{2-7} \cline{9-14}
& \multicolumn{2}{c}{\textbf{Unconstrained}} & \multicolumn{2}{c}{\textbf{Dilution-Constrained}} & \multicolumn{2}{c}{\textbf{Sample-Constrained}} & & \multicolumn{2}{c}{\textbf{Unconstrained}}                & \multicolumn{2}{c}{\textbf{Dilution-Constrained}}    & \multicolumn{2}{c}{\textbf{Sample-Constrained}}    \\
\cline{2-3} \cline{4-5} \cline{6-7} \cline{9-10} \cline{11-12} \cline{13-14}
\boldmath{$m$}& \multicolumn{1}{c}{\textbf{Top-k}} & \multicolumn{1}{c}{\textbf{Top-2k}} & \multicolumn{1}{c}{\textbf{Top-k}} & \multicolumn{1}{c}{\textbf{Top-2k}} & \multicolumn{1}{c}{\textbf{Top-k}} & \multicolumn{1}{c}{\textbf{Top-2k}}  & \boldmath{$m$} & \multicolumn{1}{c}{\textbf{Top-k}} & \multicolumn{1}{c}{\textbf{Top-2k}} & \multicolumn{1}{c}{\textbf{Top-k}} & \multicolumn{1}{c}{\textbf{Top-2k}} & \multicolumn{1}{c}{\textbf{Top-k}} & \multicolumn{1}{c}{\textbf{Top-2k}} \\ \hline
162                                 &$4.0\%$                                                &$73.5\%$                                             &$2.5\%$                                           &$74.0\%$                                                 &$7.0\%$                                             &$90.5\%$                                           
&300 
&0&$26.0\%$
&0&$26.5\%$
&0&$58.0\%$\\
189                                 &$18.0\%$                                                 &\boldmath{$99.0\%$}                                             &$22.5\%$                                           &$98.0\%$                                                 &$55.0\%$                                             &\boldmath{$100.0\%$}                                       
&360
&$0.5\%$&$93.0\%$
&0&$90.0\%$
&$9.5\%$&\boldmath{$99.0\%$}\\
216                                 &$50.0\%$                                                 &$100.0\%$                                             &$54.5\%$                                           &\boldmath{$99.5\%$}                                                 &$82.5\%$                                             &$99.5\%$                                           
&420
&$9.5\%$&\boldmath{$99.5\%$}
&$8.0\%$&\boldmath{$99.5\%$}
&$54.0\%$&$99.0\%$\\
243                                 &$71.5\%$                                                 &$100.0\%$                                             &$79.5\%$                                           &$99.5\%$                                                 &$94.5\%$                                             &$100.0\%$                                           
&480
&$36.0\%$&$99.5\%$
&$42.0\%$&$99.5\%$
&$92.5\%$&$99.0\%$\\
270                                 &$89.5\%$                                                 &$100.0\%$                                             &$89.0\%$                                           &$99.5\%$                                                 &\boldmath{$100.0\%$}                                             &$100.0\%$                                           
&540
&$72.0\%$&$99.5\%$
&$74.0\%$&$99.5\%$
&$95.0\%$&$99.0\%$\\
297                                   &$96.5\%$                                                 &$100.0\%$                                            &$95.5\%$                                           &$99.5\%$                                                 &$100.0\%$                                             &$100.0\%$                                           
&600
&$88.5\%$&$99.5\%$
&$92.0\%$&$99.5\%$
&\boldmath{$99.0\%$}&$99.5\%$\\
324                                   &$98.5\%$                                               &$100.0\%$                                             &\boldmath{$99.0\%$}                                           &$99.5\%$                                                 &$100.0\%$                                             &$100.0\%$                                           
&660
&$96.0\%$&$99.5\%$
&$97.5\%$&$99.5\%$
&$99.5\%$&$99.5\%$\\
351                                   &\boldmath{$99.5\%$}                                                 &$100.0\%$                                             &$99.0\%$                                           &
$99.5\%$                                                 &$99.5\%$                                             &$99.5\%$                                           
&720
&\boldmath{$99.5\%$}&$99.5\%$
&\boldmath{$99.0\%$}&$99.5\%$
&$99.5\%$&$99.5\%$ \\\hline
\end{tabular}}

\vspace{6pt} 
\small
\noindent For $n=5000$ and $k=10$, the optimal pooling parameters for three pooling designs are $p=0.0625$, $P_r=312$ and $P_c=19$, respectively. For $n=10000$ and $k=20$, the optimal parameters are $p=0.0312$, $P_r=312$ and $P_c=19$, respectively. For each design and sampling strategy, the minimum number of pooled tests at which the success rate reaches 99\% is marked in bold.
\vspace{-0.2cm}
\end{table*}

\begin{figure*}[!ht]
    \centering
    \captionsetup{position=below,justification=centering,singlelinecheck=false}
    \subcaptionbox{Dilution-constrained setting}[0.88\textwidth]{    
        \includegraphics[width=0.88\textwidth]{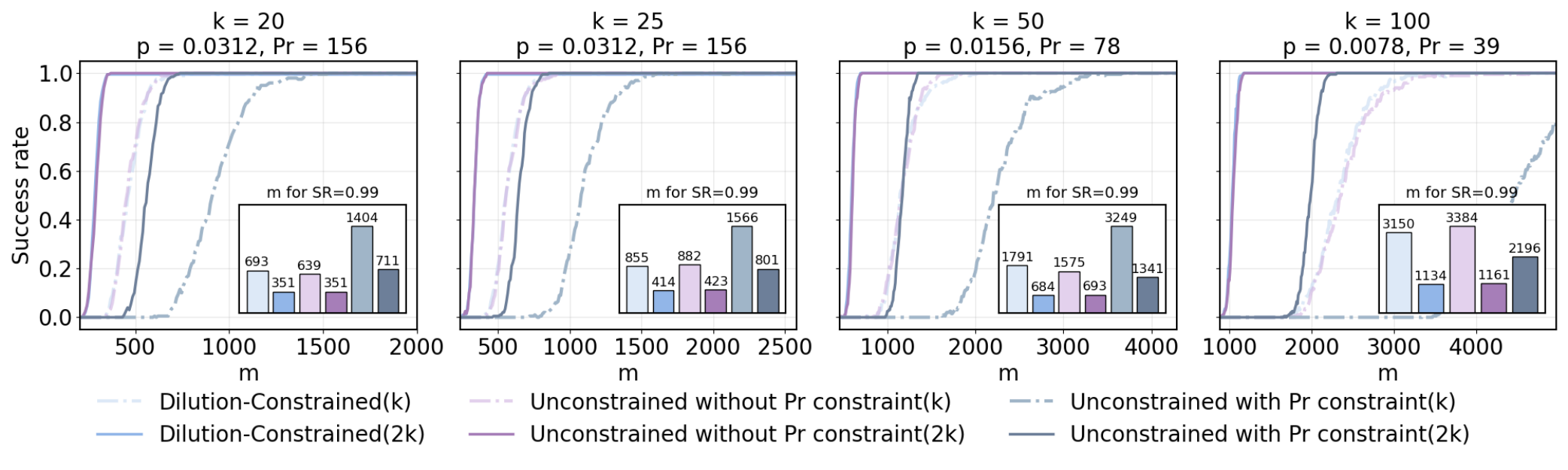}
        }
    \subcaptionbox{Sample-constrained setting}[0.88\textwidth]{
        \includegraphics[width=0.88\textwidth]{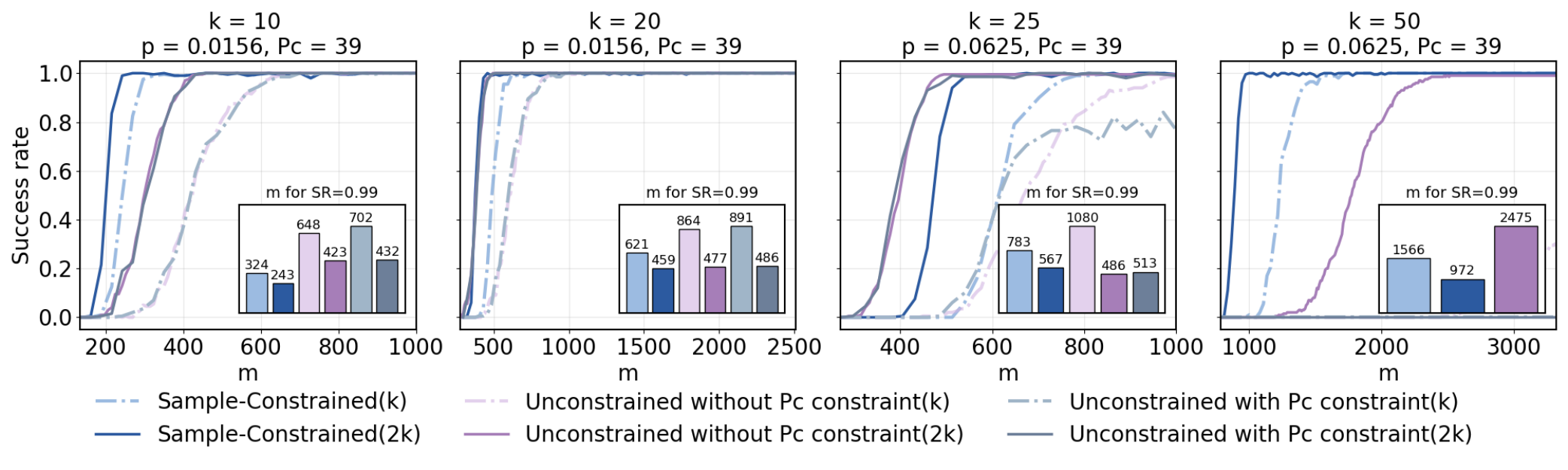}
        }
    \captionsetup{position=below,justification=justified,singlelinecheck=false}
    \caption{\textbf{Performance comparison of constrained and unconstrained designs.} Recovery success rate versus the number of pooled tests $m$ for $n=5000$ under (a) dilution-constrained and (b) sample-constrained settings. Three pooling designs are compared: the unconstrained design, the constrained design, and the unconstrained design with imposed constraints. Insets present minimum number of pooled tests required for a $99\%$ success rate.}
    \vspace{-0.6cm}
    \label{fig:combinatorial_success_rate_optimal_p_constraint_compare}
\end{figure*}

A key advantage of LoSc is the substantial performance improvement achieved by the Top-2k strategy with only a slight increase in individual tests. This over-selection strategy expands the estimated set by $k$ additional candidates and requires at most $k$ extra individual tests, yet delivers a substantial improvement in the recovery success rate. For example, when $n=5000$, $k=10$ and $m=189$, the Top-2k strategy achieves over 98\% recovery for all pooling designs, whereas the Top-k strategy reaches only $18\%$ under the unconstrained design, $22.5\%$ under the dilution-constrained design, and $55\%$ under the sample-constrained design. In the scenario with $n=10000$, $k=20$, the success rate under the Top-2k strategy increases from 0.5\% to 93\% at $m=360$ under the unconstrained design, requiring only $k=20$ additional individual tests. Similar performance improvements are observed for both dilution-constrained and sample-constrained designs.

\begin{figure*}[!ht]
\centering
    \captionsetup{position=below,justification=centering,singlelinecheck=false}
    \subcaptionbox{Comparison of testing efficiency between LoSc and combinatorial methods}[0.88\textwidth]{
        \includegraphics[width=0.88\textwidth]{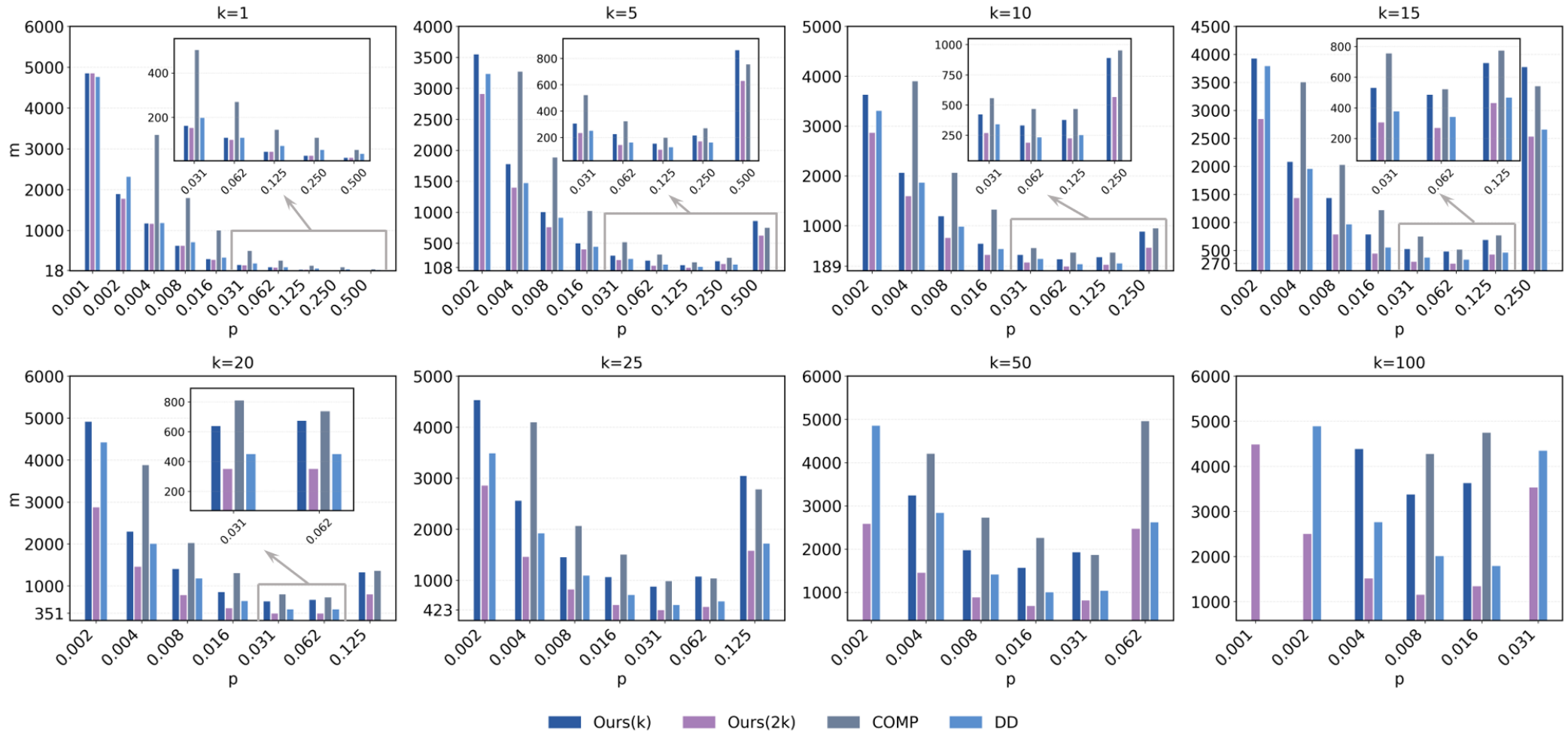}
        }
    \subcaptionbox{Comparison of computational cost between LoSc and CS-based method}[0.85\textwidth]{
        \includegraphics[width=0.85\textwidth]{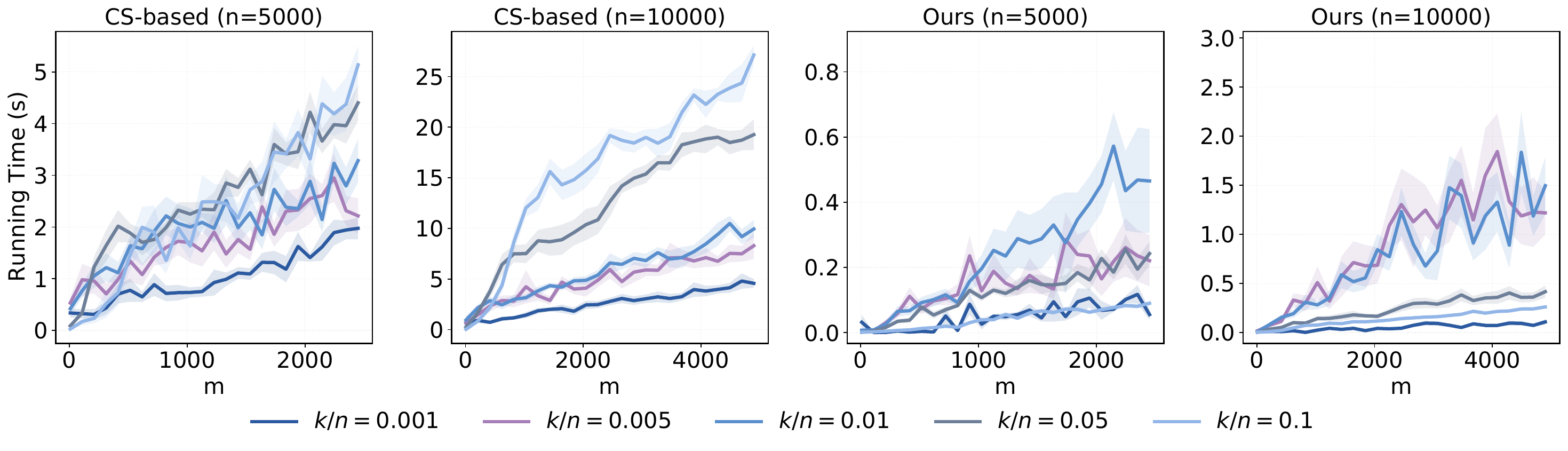}
        }
    \captionsetup{position=below,justification=justified,singlelinecheck=false}
\caption{\textbf{Comparison of decoding algorithms.} (a) Minimum number of pooled tests required for a $99\%$ success rate versus pooling parameter $p$ for LoSc under Top-k and Top-2k strategies, COMP, DD. Each panel corresponds to a different value of $k$ and $n=5000$. (b) The running times are evaluated across different values of $n$ and $k$ under the unconstrained design with $p=0.25$. The shadows near each curve represent the $90\%$ confidence interval.}
\vspace{-0.6cm}
\label{fig:algorithm_compare}
\end{figure*}

\subsection{Performance Advantage of Constrained Pooling Designs under Practical Testing Limitations}
To demonstrate the effectiveness of the proposed constrained pooling designs under practical laboratory conditions, we compare the unconstrained design with dilution-constrained and sample-constrained designs in realistic laboratory conditions. Specifically, accounting for dilution effects, pools with size exceeding $P_r$ are deemed invalid and excluded from the decoding process. Similarly, under the sample usage constraint, each specimen is assigned to at most $P_c$ pools, with any excess allocations removed. As shown in Fig.~\ref{fig:combinatorial_success_rate_optimal_p_constraint_compare}, directly imposing practical constraints on the unconstrained design degrades recovery performance, whereas pooling designs that explicitly incorporate dilution or sample usage constraints achieve superior and more robust performance. For example, when dilution effects are taken into account, the unconstrained design with imposed constraint typically requires nearly twice the number of pooled tests to achieve comparable recovery performance. These results highlight the necessity of constraint-aware pooling designs to ensure both practical feasibility and reliable detection in large-scale screening applications.

\subsection{Test Efficiency and Computational Advantages of LoSc over CS-based and Combinatorial Methods} 
LoSc achieves high testing efficiency with a near-minimal number of pooled tests, while substantially reducing computational cost. We first evaluate the testing efficiency of LoSc in comparison with combinatorial methods. Fig.~\ref{fig:algorithm_compare}(a) presents the minimum number of pooled tests required for a $99\%$ success rate under the unconstrained design. The results show that LoSc under the Top-2k strategy consistently requires the fewest pooled tests to reach the target success rate. Under the Top-k strategy, LoSc exhibits performance comparable to DD for small values of $k$ (e.g., $k\leq 20$). COMP performs the worst and requires substantially more tests. The performance advantage becomes more pronounced as $k$ increases: in some cases, COMP or DD needs up to twice as many tests as LoSc (under the Top-2k strategy), and COMP becomes ineffective when $k>100$ ($n=5000$). COMP only uses the information of negative pools, leading to an overly conservative estimated set, while DD confirms a sample as positive only when it is the sole member of a positive pool, demanding more tests to avoid missing. In contrast, our algorithm prioritizes samples according to their frequencies of occurrence in positive pools and incorporates an over-selection strategy, enabling identification of all positives with significantly fewer tests.
\begin{table*}[!ht]
\caption{Experimentally optimal parameters for pooling designs in LoSc ($n=5000$)}\label{tab:combinatorial_optimal pooling design_n5000}
\centering
\begin{tabular}{ccccccccccccc}
\hline
& \multicolumn{4}{c}{\textbf{Unconstrained}}                                                             & \multicolumn{4}{c}{\textbf{Dilution-Constrained}}                                                            & \multicolumn{4}{c}{\textbf{Sample-Constrained}}                                                         \\ \cline{2-5} \cline{6-9} \cline{10-13} 
                        & \multicolumn{2}{c}{\textbf{Top-k}}                   & \multicolumn{2}{c}{\textbf{Top-2k}}                  & \multicolumn{2}{c}{\textbf{Top-k}}                   & \multicolumn{2}{c}{\textbf{Top-2k}}                  & \multicolumn{2}{c}{\textbf{Top-k}}                   & \multicolumn{2}{c}{\textbf{Top-2k}}                  \\ \cline{2-3} \cline{4-5} \cline{6-7} \cline{8-9} \cline{10-11} \cline{12-13} 
                        \boldmath{$k$} & \boldmath{$p^*$} & \boldmath{$m^*$} & \boldmath{$p^*$} & \boldmath{$m^*$} & \boldmath{$P^*_r$} & \boldmath{$m^*$} & \boldmath{$P^*_r$} & \boldmath{$m^*$} & \boldmath{$P^*_c$} & \boldmath{$m^*$} & \boldmath{$P^*_c$} & \boldmath{$m^*$} \\ \hline
1                & 0.5000                     & \textbf{18}                      & 0.5000                     & \underline{18}                     & 2500                     & 18                     & 2500                    & 18                     &  6                    & 27                     &  6                    & 27                     \\
5                 &  0.1250                    &  \textbf{153}                    & 0.1250                     &  \underline{108}                    & 625                     & 162                     &  625                    &  117                    &  13
                    & 162                     &  6                    & 108                     \\
10                 &  0.0625
                    & 333                     & 0.0625
                     &  189                    &  312
                    & 324                     & 312
                     &  198                    &  19
                    & \textbf{270}                     & 6
                     &  \underline{162}                    \\
15                 & 0.0625
                     & 486                     & 0.0625
                     & 270                     & 312
                     & 459                     & 312
                     & 261                     & 13                     & \textbf{405}                     & 9
                     & \underline{216}                     \\
20                 & 0.0313
                     & 639                     & 0.0313
                     & 351                     &  156
                    &  693                    & 156
                     & 351                     & 27
                     & \textbf{540}                     &   6
                   & \underline{270}                     \\
25                 & 0.0313
                     & 882                     & 0.0313
                     & 423                    & 156
                     & 855                     & 156
                     & 414                     & 13
                     & \textbf{675}                    & 9
                     & \underline{324}                    \\
50                   &  0.0156
                    &  1575                    & 0.0156
                     &  693                    & 78
                     &  1791                    & 78
                     &  684                    & 13
                     & \textbf{1350}                     & 6                     & \underline{540}                     \\
100                   & 0.0078
                     & 3384                     & 0.0078
                     & 1161                     & 39
                     & 3150                     & 39
                     & 1134                     & 13
                     & \textbf{2619}                     & 6
                     & \underline{891}                     \\
150                   & -                     & -                     & 0.0078                     & 1602                     & -                     & -                     & 39                     & 1593                     & 19                     & \textbf{3942}                     & 4                     & \underline{1188}                     \\
200                   & -                     & -                     & 0.0039                     & 1908                     & -                     & -                     & 19                    & 1863                     & -                     & -                     & 4                     & \underline{1458}                     \\
250                   & -                     & -                     & 0.0039                     & 2142                     & -                     & -                     & 19                      & 2052                     & -                     & -                     & 4                     & \underline{1620}                     \\
500                  & -                     & -                     & 0.0020                     & 3150                     & -                     & -                     & 9                     & 3024                     & -                     & -                     & 3                     & \underline{2403}                     \\ \hline
\end{tabular}

\vspace{6pt}
\begin{minipage}{\textwidth}
\small\justifying
\noindent Each design is evaluated at its experimentally optimal pooling parameter ($p^*$, $P^*_r$ or $P^*_c$). Here, $m^*$ denotes the minimum number of pooled tests required for a 99\% recovery success rate. Optimal values of $m^*$ under the Top-k strategy are marked in bold, and optimal values of $m^*$ under the Top-2k strategy are underlined. Some results are not reported as the number of tests required exceeds the population size.
\end{minipage}
\vspace{-0.2cm}
\end{table*}

Beyond its superior testing efficiency compared with combinatorial methods, LoSc also provides a substantial computational advantage over the CS-based method. Our decoding algorithm has a computational complexity of $\mathcal{O}(km)$, while the algorithm (OMP) used in the CS-based method requires $\mathcal{O}(k^3 + 2mk^2 + nm + mk)$ operations \cite{jhang2016high}, leading to an approximately $\frac{n}{k}$-fold increase in computational cost. Thus, the computational advantage of LoSc becomes more pronounced as $n$ increases, particularly in large-scale screening applications. We further compare the running time of two methods across different values of $n$ and $k$ under the unconstrained design (Fig.~\ref{fig:algorithm_compare}(b)). Experiments are conducted on a Linux server equipped with an Intel Xeon Gold 6226R CPU (32 cores). For a fixed $n$, LoSc exhibits a significant speed advantage, achieving nearly $10\times$ faster inference than the CS-based method. For example, when $n=5000$, $k=50$ and $m=2450$, the CS-based method takes 4.07 seconds versus 0.46 seconds for LoSc. When $n=10000$, $k=100$ and $m=2454$, the CS-based method requires 10.30 seconds compared to only 1.26 seconds for LoSc. This speedup is critical for large-scale screening with limited computational resources. 

Additionally, we compare the recovery performance of LoSc and the CS-based method in terms of the number of pooled tests required for a $99\%$ success rate under their respective optimal pooling parameters. Results for LoSc under three pooling designs are summarized in Table~\ref{tab:combinatorial_optimal pooling design_n5000} and Table~\ref{tab:combinatorial_optimal pooling design_n10000} in the Supplementary Material, whereas the CS-based method is evaluated under the unconstrained design and results are summarized in Table~\ref{tab:baseline_optimal pooling design}. These results show that, although the CS-based method requires slightly fewer tests than LoSc, the number of assays required by LoSc is already very small. For example, when $n=10000$ and $k=2$, the CS-based method identifies all positives using $m=32$ pooled tests without follow-up testing, whereas LoSc under the Top-2k strategy requires $m=50$ pooled tests and $2k=4$ individual tests, resulting in only 22 additional assays. LoSc achieves comparable performance with significantly lower computational cost, making it suitable for practical high-throughput deployment.
\begin{table}[!ht]
\caption{Experimentally optimal parameters for unconstrained design in CS-based method}\label{tab:baseline_optimal pooling design}
\centering
\begin{tabular}{cccccc}
\hline
& \multicolumn{2}{c}{\boldmath{$n=5000$}}     &                                        & \multicolumn{2}{c}{\boldmath{$n=10000$}}                                             \\ \cline{2-3} \cline{5-6}
\boldmath{$k$} & \boldmath{$p^*$} & \boldmath{$m^*$} & \boldmath{$k$}& \boldmath{$p^*$}  & \boldmath{$m^*$} \\ \hline
1                                &0.5000                                                 &12                                             &2                                           &0.2240                                                 &32                                             \\
5                                 &0.1072                                                 &67                                             &10                                           &0.0498                                                 &155                                            \\
10                                 &0.0698                                                 &108                                             &20                                           &0.0338                                                 &254                                           \\
15                                 &0.0476                                                 &180                                             &30                                           &0.0284                                                 &318                                           \\
20                                 &0.0365                                                 &231                                             &40                                           &0.0222                                                 &415                                          \\
25                                 &0.0298                                                 &288                                             &50                                           &0.0186                                                 &501                                            \\
50                                   &0.0200                                                 &485                                             &100                                           &0.0110                                                 &946                                         \\
100                                   &0.0111                                                 &924                                             &200                                           &0.0064                                                 &1778                                           \\
150                                   &0.0085                                                 &1326                                             &300                                           &-                                                 &-                                         \\
200                                   &0.0065                                                 &1722                                             &400                                           &-                                                 &-                                          \\
250                                   &0.0053                                                 &2209                                             &500                                           &-                                                 &-                                         \\
500                                  &0.0043                                                 &4334                                             &1000                                           &-                                                 &-                                          \\ \hline
\end{tabular}

\vspace{6pt}
\begin{minipage}{\linewidth}
\small\justifying
\noindent Here, $m^*$ denotes the minimum number of pooled tests required for a $99\%$ success rate. Some results for $n=10000$ are not reported due to high computational complexity.
\end{minipage}
\vspace{-0.6cm}
\end{table}

\begin{figure}[!t]
    \centering
    \captionsetup{position=below,justification=centering,singlelinecheck=false}
    \subcaptionbox{Unconstrained design}[1\linewidth]{    
        \includegraphics[width=1\linewidth]{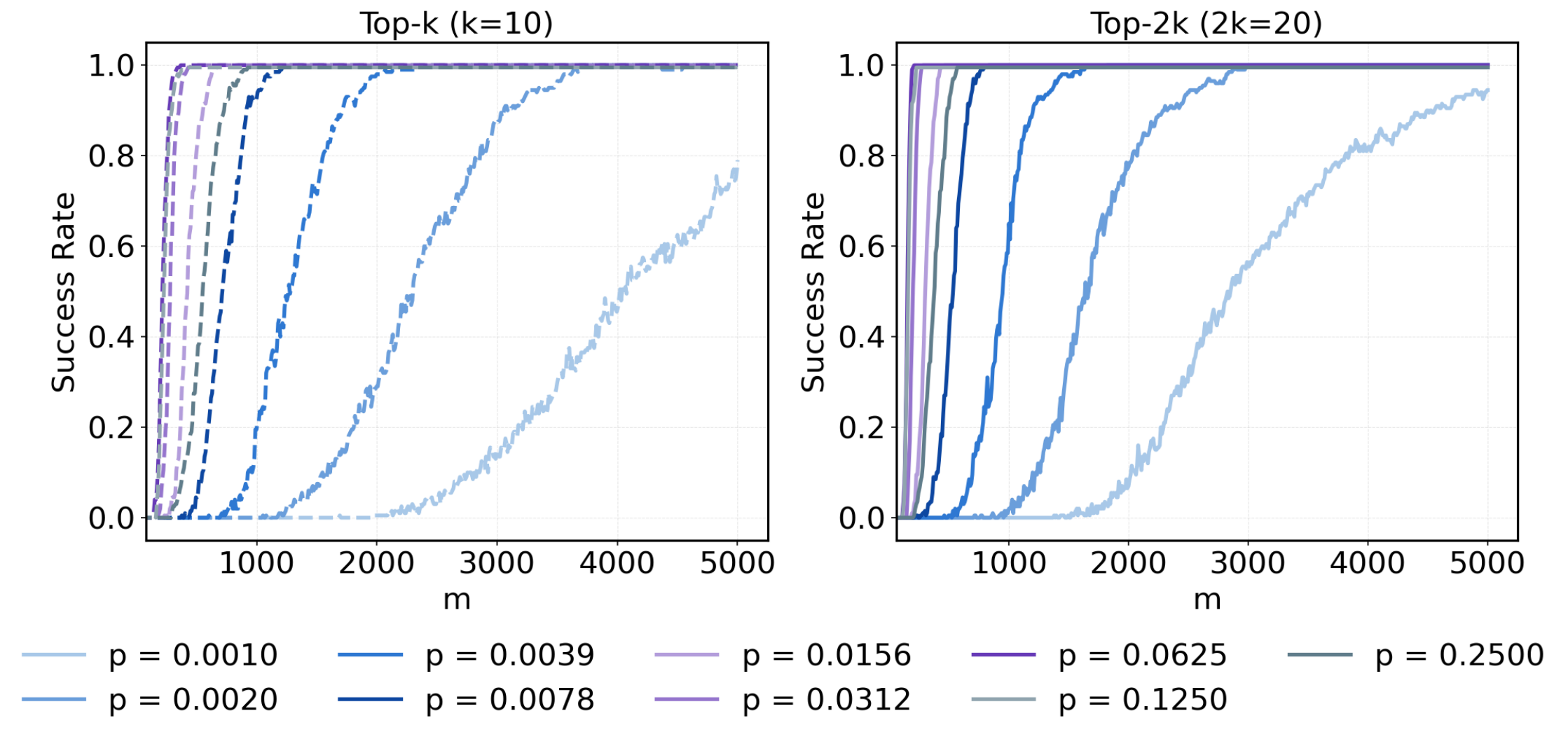}
        }
    \quad
    \subcaptionbox{Dilution-constrained design}[1\linewidth]{
        \includegraphics[width=1\linewidth]{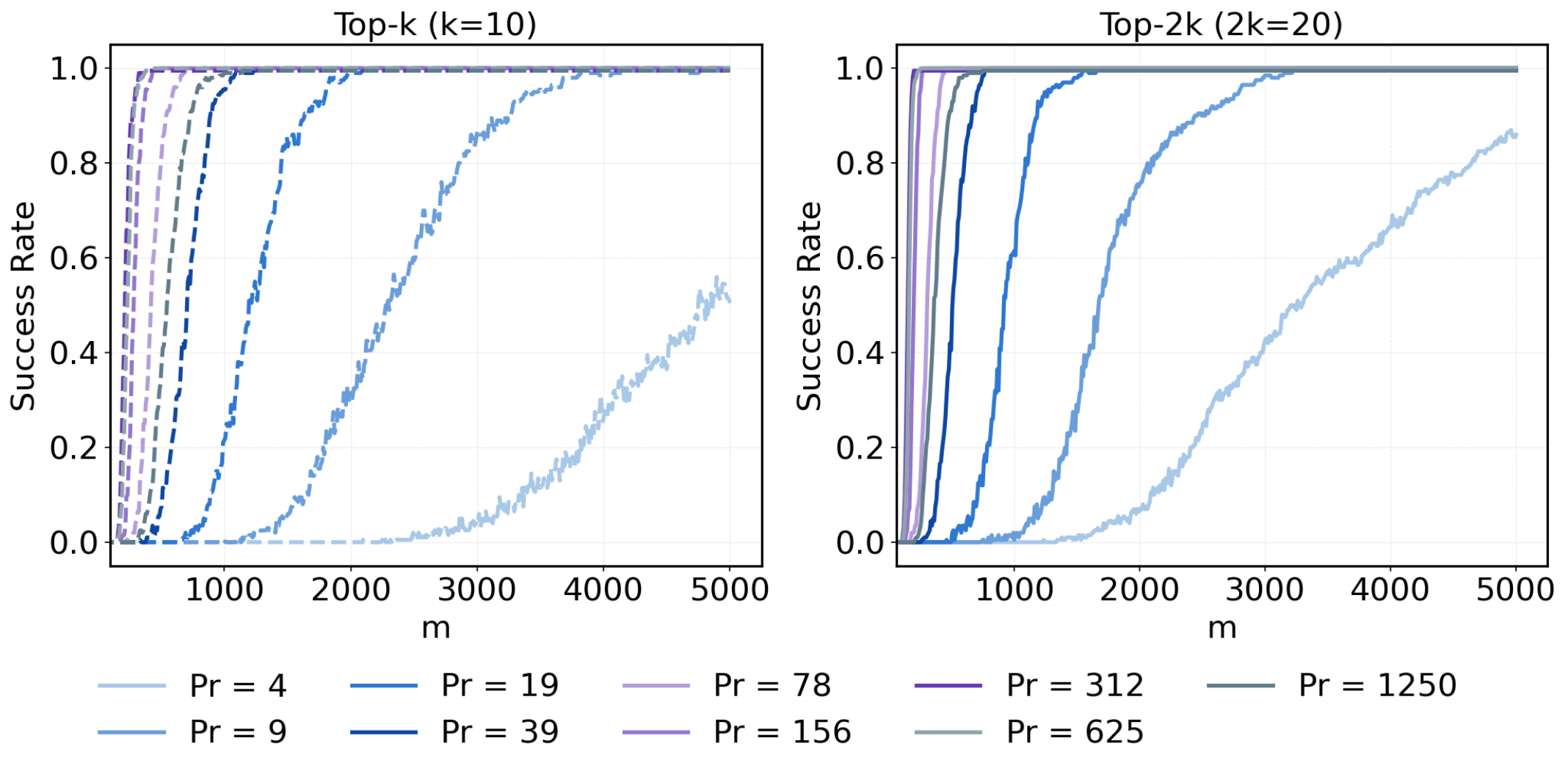}
        }
    \quad
    \subcaptionbox{Sample-constrained design}[1\linewidth]{
        \includegraphics[width=1\linewidth]{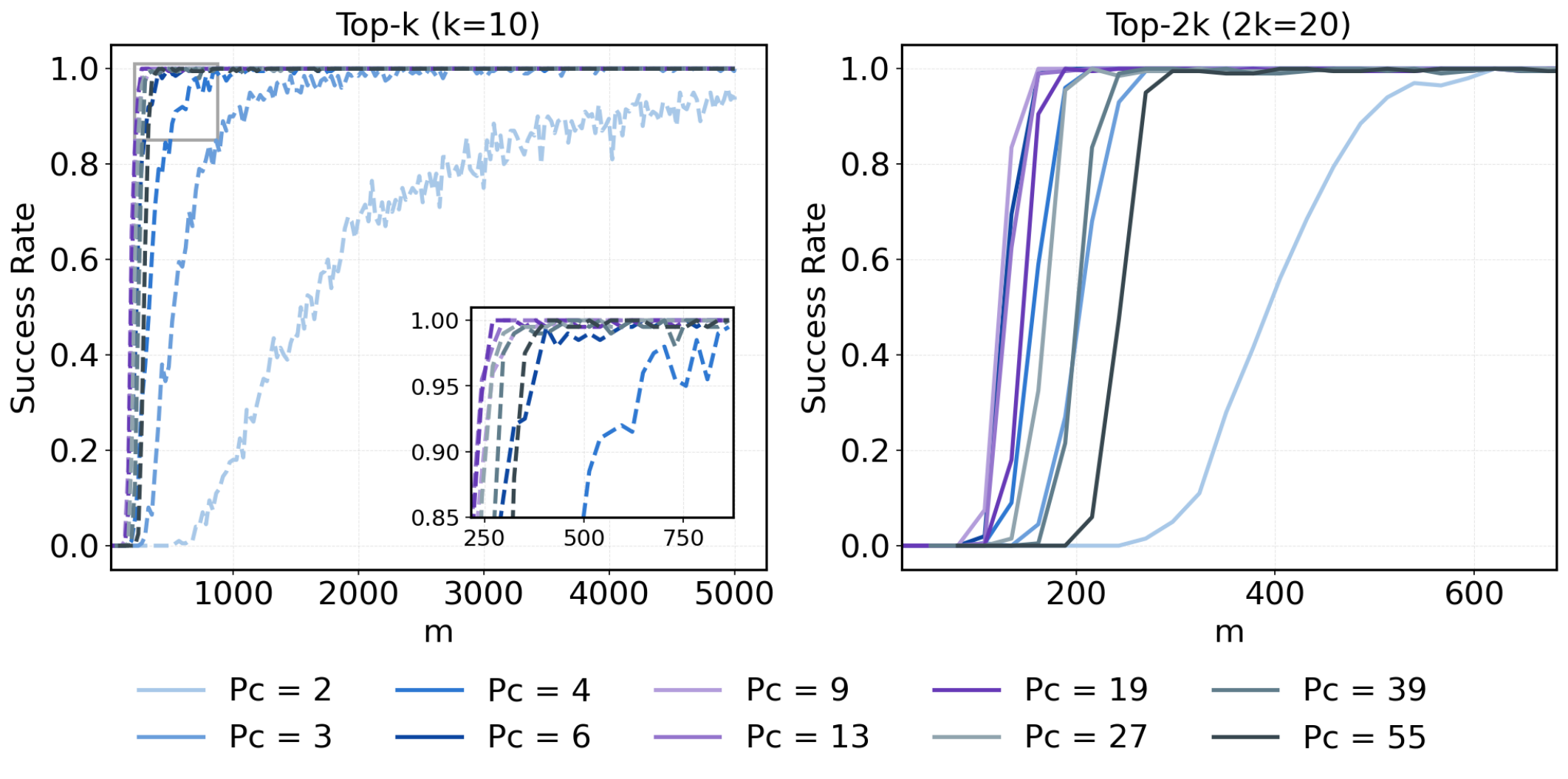}
        }
    \captionsetup{position=below,justification=justified,singlelinecheck=false}
    \caption{\textbf{Impact of pooling parameters on recovery performance.} Recovery success rate as a function of $m$ under different pooling parameters when $n=5000$ and $k=10$.}
    \vspace{-0.6cm}
    \label{fig:combinatorial_success_rate_n5000_kr0.002}
\end{figure}

\begin{figure*}[!ht]
\centering
\includegraphics[width=0.88\textwidth]{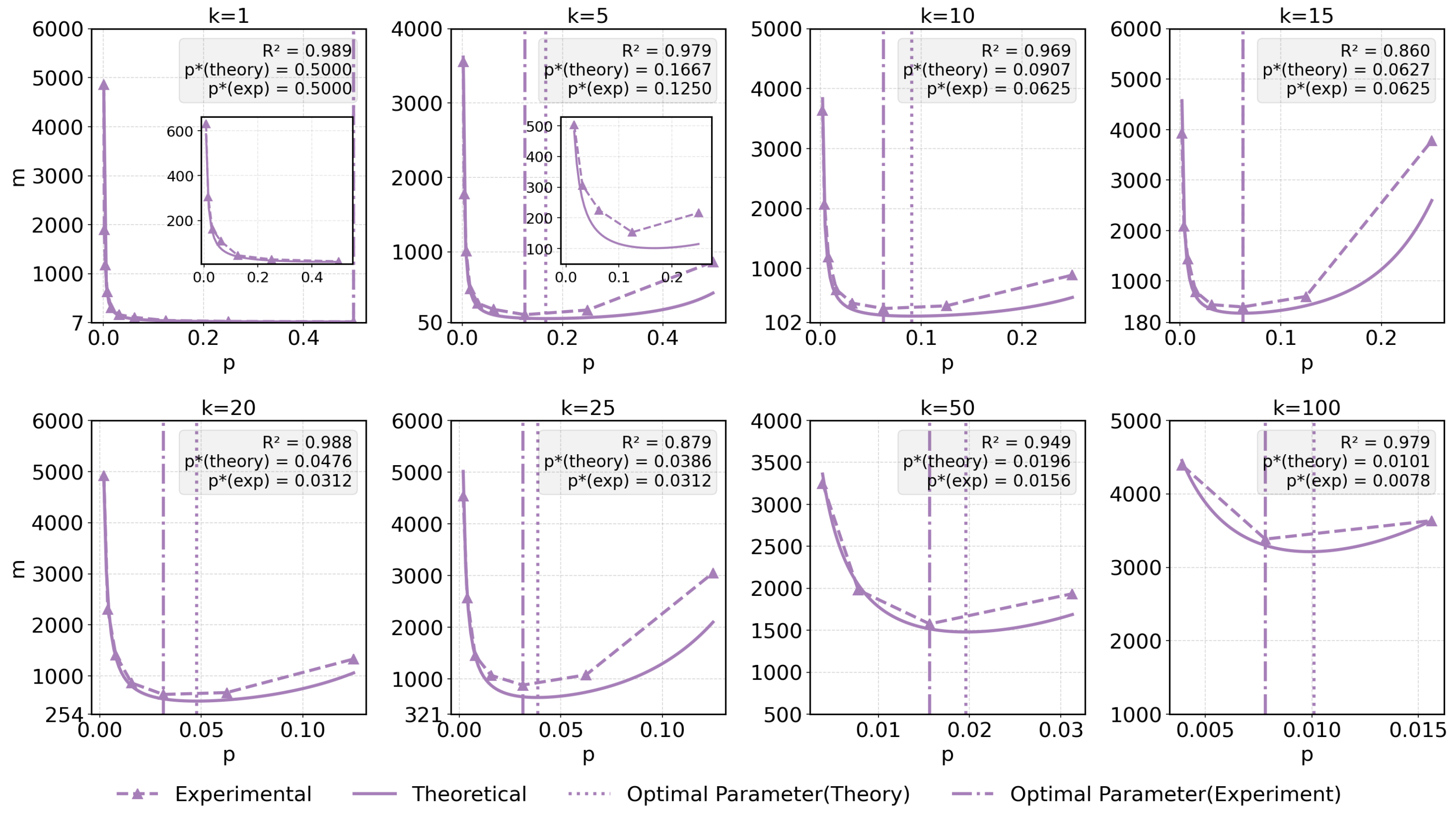}
\caption{\textbf{Experimental and theoretical results of \boldmath{$m^*(p)$} for unconstrained design.} $m^*(p)$ is the function of $p$ under the Top-k strategy when $n=5000$. Vertical lines mark the theoretically and experimentally optimal pooling parameters.}
\vspace{-0.2cm}
\label{fig:combinatorial_fit_n5000_bernoulli_k}
\end{figure*}

\begin{figure*}[!ht]
\centering
\includegraphics[width=0.88\textwidth]{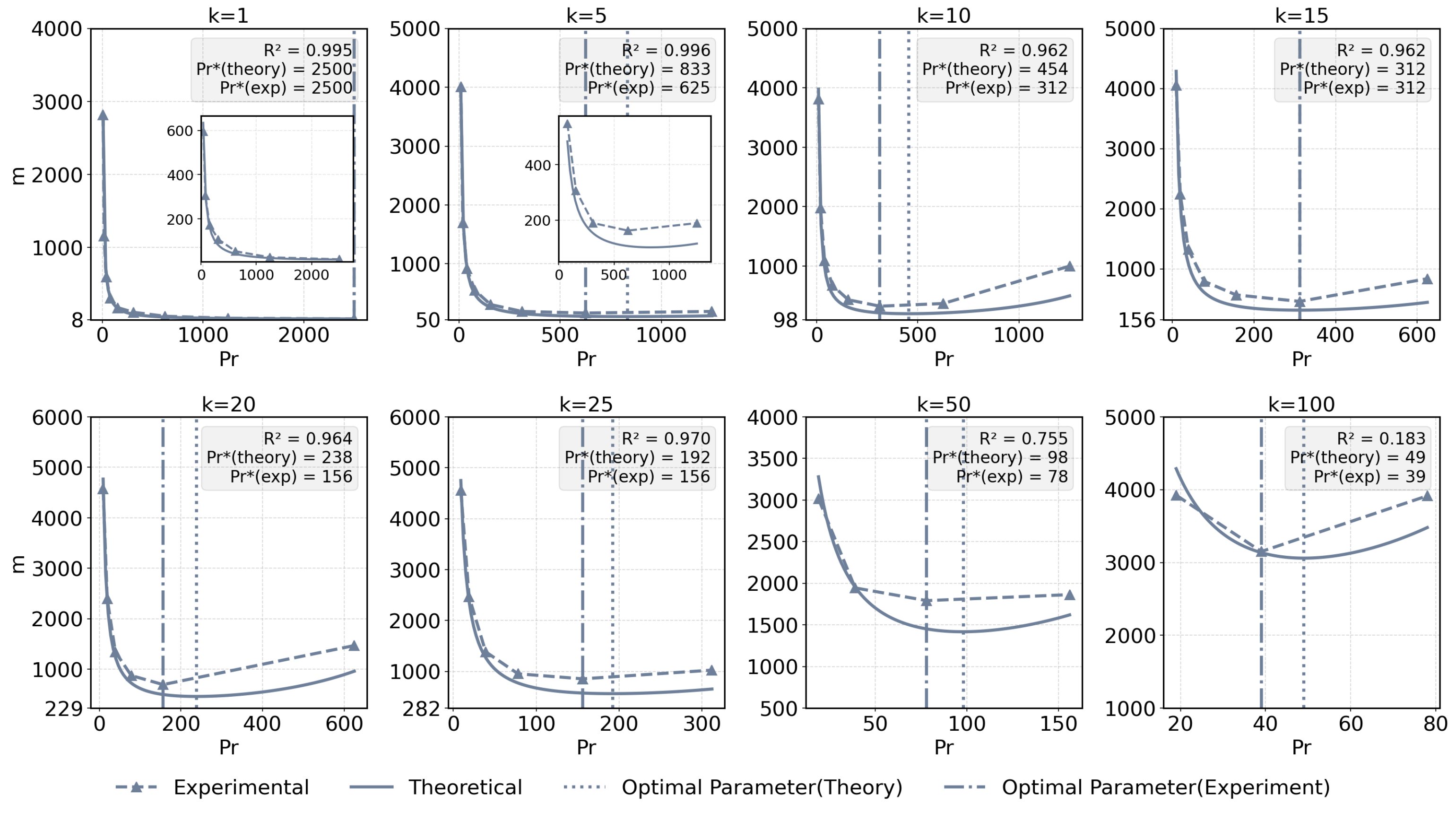}
\caption{\textbf{Experimental and theoretical results of \boldmath{$m^*(P_r)$} for dilution-constrained design.} $m^*\left(P_r\right)$ is the function of $P_r$ under the Top-k strategy when $n=5000$. Vertical lines mark the theoretically and experimentally optimal pooling parameters.}
\vspace{-0.6cm}
\label{fig:combinatorial_fit_n5000_row_k}
\end{figure*}

\begin{figure*}[!ht]
\centering
\includegraphics[width=0.88\textwidth]{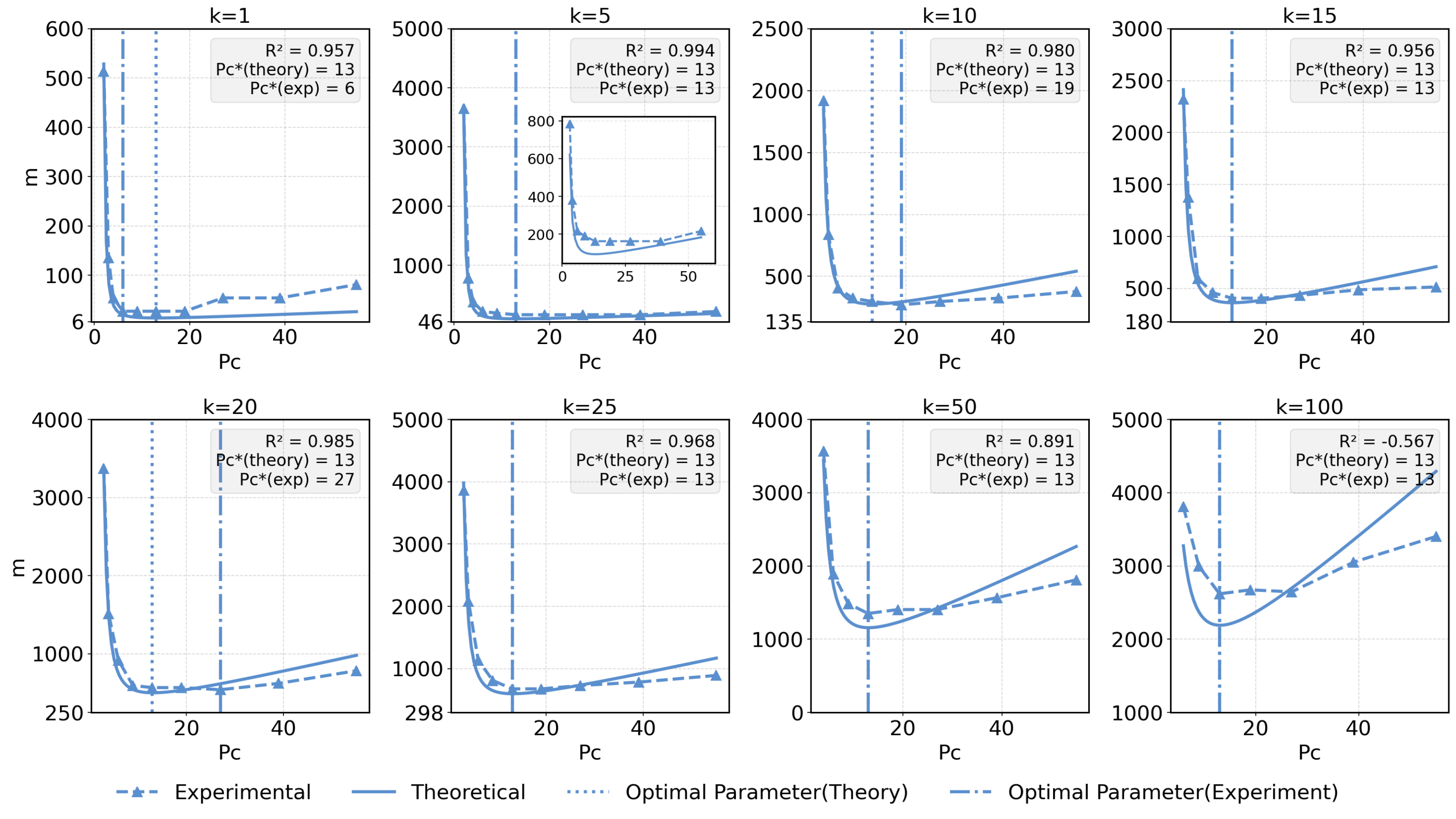}
\caption{\textbf{Experimental and theoretical results of \boldmath{$m^*\left(P_c\right)$} for sample-constrained design.} $m^*\left(P_c\right)$ is the function of $P_c$ under the Top-k strategy when $n=5000$. Vertical lines mark the theoretically and experimentally optimal pooling parameters.}
\vspace{-0.6cm}
\label{fig:combinatorial_fit_n5000_col_k}
\end{figure*}

\subsection{Impact of Pooling Parameters} To investigate how pooling parameters ($p$, $P_r$, and $P_c$) influence the recovery performance, we evaluate our method by varying the corresponding pooling parameter for each design (Fig.~\ref{fig:combinatorial_success_rate_n5000_kr0.002} and Figs.~\ref{fig:combinatorial_success_rate_n5000_kr0.01}-~\ref{fig:combinatorial_success_rate_n10000_kr0.01} in Supplementary Material). The results show that, across three pooling designs, recovery performance exhibits a non-monotonic trend with respect to the pooling parameters. In the unconstrained design, recovery performance is controlled by the pooling parameter $p$. In the very sparse pooling regime, the probability that a sample is assigned to at least one pool, given by $1-(1-p)^m$, is extremely low, resulting in many positives untested and poor recovery performance. As $p$ increases, more samples are included in pools, enhancing the elimination of negatives and recovery performance. However, beyond a critical point, further increasing $p$ causes a large number of pools to contain positives, drastically reducing the number of negative pools and thereby degrading performance. This trade-off implies the existence of an optimal pooling parameter $p^*$. A similar phenomenon is observed in the row-constrained design. By limiting the number of samples per pool, $P_r$ controls the sparsity of the pooling matrix, yielding the same trade-off as in the unconstrained design. The sample-constrained design is controlled by the pooling parameter $P_c$, which determines the number of times that each sample can be used. As $P_c$ increases, each sample is included in more pools, providing more information for reliable recovery and improving performance. However, when $P_c$ becomes excessively large, the probability that positive and negative samples co-occur in the same pools increases, making it difficult to distinguish positives from negatives. Consequently, the best performance is achieved at an intermediate value of $P_c$.

To verify the optimal pooling parameters given by Theorems~\ref{thm:unconstrained}-\ref{thm:sample_constrained}, we examine the minimum number of pooled tests $m^*$ required for a $99\%$ success rate as functions of pooling parameters under different pooling designs. For the unconstrained design, $m^*(p)$ exhibits a pronounced U-shape with its minimum attained near the theoretically optimal pooling parameter (Fig.~\ref{fig:combinatorial_fit_n5000_bernoulli_k} and Fig.~\ref{fig:combinatorial_fit_n10000_bernoulli_k} in the Supplementary Material). To further characterize this behavior, we fit the experimental data $m^*(p)$ using the expression in Conjecture~\ref{unconstrained_conjecture}. The fitted curves are in close agreement with the experimental results across different values of $n$ and $k$, indicating that the conjecture accurately captures the scaling behavior of $m^*(p)$. 

A similar U-shaped trend is observed in the dilution-constrained design (Fig.~\ref{fig:combinatorial_fit_n5000_row_k} and Fig.~\ref{fig:combinatorial_fit_n10000_row_k} in the Supplementary Material), where $m^*\left(P_r\right)$ attains its minimum near the theoretically optimal pooling parameter. We further fit the experimental data $m^*\left(P_r\right)$ using the expression given in Conjecture~\ref{dilution_constrained_conjecture}. The close agreement confirms that the conjecture correctly describes the relationship between $m^*\left(P_r\right)$ and the pooling parameter $P_r$. 

For the sample-constrained design, as shown in Fig.~\ref{fig:combinatorial_fit_n5000_col_k} and Fig.~\ref{fig:combinatorial_fit_n10000_col_k} in the Supplementary Material, the minimum of $m^*\left(P_c\right)$ is observed approximately at the theoretically optimal pooling parameter, and the experimental results closely follow the behavior predicted by Theorem~\ref{thm:sample_constrained} across different settings, supporting its applicability in practice.

\begin{table*}[!t]
\caption{Comparison between theoretical and experimental optimal parameters under the Top-k strategy ($n=5000$)}\label{tab:combinatorial_optimal pooling design_n5000_compare}
\centering
\begin{tabular}{ccccccccccccc}
\hline
& \multicolumn{4}{c}{\textbf{Unconstrained}}                                                             & \multicolumn{4}{c}{\textbf{Dilution-Constrained}}                                                            & \multicolumn{4}{c}{\textbf{Sample-Constrained}}                                                         \\ \cline{2-5} \cline{6-9} \cline{10-13} 
\cline{2-3} \cline{4-5} \cline{6-7} \cline{8-9} \cline{10-11} \cline{12-13}
& \multicolumn{2}{c}{\textbf{Experiment}}  & \multicolumn{2}{c}{\textbf{Theory}} & \multicolumn{2}{c}{\textbf{Experiment}}  & \multicolumn{2}{c}{\textbf{Theory}} & \multicolumn{2}{c}{\textbf{Experiment}}  & \multicolumn{2}{c}{\textbf{Theory}}  \\
\cline{2-3} \cline{4-5} \cline{6-7} \cline{8-9} \cline{10-11} \cline{12-13}

\boldmath{$k$} & \boldmath{$p^*$} & \boldmath{$m^*$} & \boldmath{$p^*$} & \boldmath{$m^*$} & \boldmath{$P^*_r$} & \boldmath{$m^*$} & \boldmath{$P^*_{r}$} & \boldmath{$m^*$} & \boldmath{$P_c^*$} & \boldmath{$m^*$} & \boldmath{$P_c^*$} & \boldmath{$m^*$} \\ \hline
1                & 0.5000                     & 18                      & 0.5000                     & 46                     & 2500                     & 18                     & 2500                    & 46                     &  6                    & 27                     &  13                    &36                     \\
5                 &  0.1250                    &  153                    & 0.1667                     &  190                    & 625                     & 162                     &  833                    &  190                    &  13
                    & 162                     &  13                    & 179                     \\
10                 &  0.0625
                    & 333                     & 0.0909
                     &  368                    &  312
                    & 324                     & 454
                     &  368                    &  19
                    & 270                     & 13
                     &  357                    \\
15                 & 0.0625
                     & 486                     & 0.0625
                     & 547                     & 312
                     & 459                     & 312
                     & 546                     & 13                     & 405                     & 13
                     & 535                     \\
20                 & 0.0313
                     & 639                     & 0.0476
                     & 725                     &  156
                    &  693                    & 238
                     & 723                     & 27
                     & 540                     &   13
                   & 714                     \\
25                 & 0.0313
                     & 882                     & 0.0385
                     &  903                    & 156
                     & 855                     & 192
                     &  901                     & 13
                     & 675                    & 13                     & 892                    \\
50                   &  0.0156
                    &  1575                    & 0.0196
                     &  1794                    & 78
                     &  1791                    & 98
                     &  1785                    & 13
                     & 1350                     & 13                     & 1783                     \\
100                   & 0.0078
                     & 3384                     & 0.0099
                     & 3573                     & 39
                     & 3150                     & 49
                     & 3537                     & 13
                     & 2619                     & 13
                     & 3562                    \\ \hline
\end{tabular}
\vspace{-0.6cm}
\end{table*}
Furthermore, we compare the theoretically predicted optima with the experimentally observed optima, which are summarized in Table~\ref{tab:combinatorial_optimal pooling design_n5000_compare} and Table~\ref{tab:combinatorial_optimal pooling design_n10000_compare} in the Supplementary Material. Theoretical results for the unconstrained and dilution-constrained designs are obtained from Conjectures~\ref{unconstrained_conjecture}-\ref{dilution_constrained_conjecture}, and those for the sample-constrained design are obtained from Theorem~\ref{thm:sample_constrained}. The experimental and theoretical results are of the same order of magnitude with only small deviations, indicating that the theoretical results provide practical guidance for selecting near-optimal pooling parameters and conservative estimates of the testing requirements. In practice, to mitigate dilution effects and enhance detection reliability, particularly in dense-pooling or low-prevalence scenarios where viral loads are highly diluted, it is recommended to adopt slightly more conservative pooling parameters.

\section{Conclusion}\label{sec:conclusion}
In this paper, we have presented the Logic Screening method (LoSc), a novel framework for large-scale screening in the very early stages of an outbreak. LoSc integrates two practical constraints into pooling designs, the dilution constraint on the pool size and the sample usage constraint, to ensure reliable identification under realistic laboratory conditions. Furthermore, LoSc introduces an efficient decoding algorithm with an over-selection strategy, enabling substantial reductions in the number of pooled tests. For all pooling designs, we established mathematical guarantees on the number of pooled tests required for accurate identification, and proposed refined conjectures validated by extensive simulations. Our results demonstrate that LoSc achieves testing efficiency comparable to the CS-based method, while significantly outperforming combinatorial algorithms. Moreover, the computational complexity of CS-based method is approximately $\frac{n}{k}$-times higher than that of LoSc, highlighting the superior efficiency and scalability of our method. Together, these results establish LoSc as a theoretically grounded, computationally efficient, and automatically deployable framework for large-scale screening. Future work will focus on developing robust algorithms capable of handling noise and testing errors, as well as exploring optimized pooling strategies.

\appendices
\section{Proof of Theorem~\ref{thm:unconstrained}}\label{pf:unconstrained}
Before the proof of Theorem~\ref{thm:unconstrained}, we first introduce the Hoeffding's inequality.
\begin{lemma}\label{hoeffding}
\rm{(Hoeffding's Inequality \cite{book})} Let $X_{1},\dots,X_{n}$ be independent random variables with $X_{i}\in [a_{i}, b_{i}]$ almost surely, $i=1,\ldots,n$. Then we have, for every $t>0$, 
\begin{align}
    \mathbb{P}\left(\sum\limits_{i = 1}^n {\left( {{X_i} - \mathbb{E}{\left[X_i\right]}} \right)}  \geq t\right)\leq\exp \left( { - {{2{t^2}}}/{{\sum\limits_{i = 1}^n {{{\left( {{a_i} - {b_i}} \right)}^2}} }}} \right).
\end{align}
Similar bounds apply to the lower deviation $\sum\limits_{i = 1}^n {\left( {{X_i} - \mathbb{E}\left[X_i\right]}\right)}  \leq -t$ as well as the two-sided deviation $\left| {\sum\limits_{i = 1}^n {\left( {{X_i} - \mathbb{E}{\left[X_i\right]}} \right)} } \right| \geq t$, with an additional factor of two in the latter case \cite{JMLR:v22:19-479}.
\end{lemma}

\begin{proof}[Proof of Theorem \ref{thm:unconstrained}.] 
The decoding algorithm first eliminates the samples in negative pools and obtains a candidate set $C\subseteq \{1,\ldots,n\}$, then selects samples according to the frequency of occurrence in positive pools. Let $B^{+}_{i}=\left| {\left\{ {l:{\tilde y_l} = 1\,and\,{\Phi _{l,i}} = 1} \right\}} \right|$ be the number of positive pools containing the positive sample $i\in S$, and $B^{-}_{j}=\left| {\left\{ {l:{\tilde y_l} = 1\,and\,{\Phi _{l,j}} = 1} \right\}} \right|$ be the corresponding count for a negative sample $j\notin S$. For each negative sample $j$, we define $E_j$ as the event of $j\in C$, which occurs with probability $\left(1-p(1-p)^k\right)^m$. 

In the unconstrained design, if the sample $i$ is positive and is assigned to the pool $l$, then $\tilde y_{l}=1$. Thus,
\begin{align}
    B^{+}_{i}=\sum\limits_{l = 1}^m {\mathbb{I}\left( {{\Phi _{l,i}} = 1} \right)} \sim Binomial\left( {m,p} \right),
\end{align}
with $\mathbb{E}\left[B^{+}_{i}\right]=mp$, where $\mathbb{I}(\cdot)$ is the indicator function. For convenience, set $\mu_{S}=\mathbb{E}\left[B^{+}_{i}\right]$. For a negative sample $j$ assigned to the pool $l$, $\tilde y_{l}=1$ if and only if the pool $l$ contains at least one positive sample, which occurs with probability $1-(1-p)^k$. Thus, for a negative sample $j\in C$, the conditional distribution of $B^{-}_{j}$ is
\begin{align}
    B^{-}_{j}\mid E_j \sim Binomial\left( {m,q} \right),
\end{align}
with $\mathbb{E}\left[B^{-}_{j}\mid E_j\right]=mq$, where $q = \frac{p\left(1-(1-p)^k\right)}{1-p(1-p)^k}$. Set $\mu_{S^c}=\mathbb{E}\left[B^{-}_{j}\mid E_j\right]$ for simplicity.

\paragraph{Top-k sample selecting strategy} There exist false negatives if $\mathop {\min }\limits_{i \in {S}} \, {{B}^{+}_i} \leq \mathop {\max }\limits_{j \in {C\backslash S}} \,{B^{-}_j}$. The probability of this event satisfies
\begin{align}\label{bernoulli_topk_ineq}
    \begin{gathered}
        \mathbb{P} \left( {\mathop {\min }\limits_{i \in {S}} \, {B^{+}_i} \leq \mathop {\max }\limits_{j \in {C\backslash S}} \,{B^{-}_j}} \right) \hfill \\
        \leq \mathbb{P} \left( {\mathop {\min }\limits_{i \in {S}} \,  {B^{+}_i} \leq T} \right) + \mathbb{P} \left( {\mathop {\max }\limits_{j \in {C\backslash S}} \,{B^{-}_j} \geq T} \right), \\
    \end{gathered}
\end{align}
for any threshold $T>0$. We choose $T={(\mu_S+\mu_{S^c})}/{2}$. Applying Lemma~\ref{hoeffding}, the first term on the right-hand side of inequality~(\ref{bernoulli_topk_ineq}) is bounded by
\begin{align}
    \mathbb{P}\left( {\mathop {\min }\limits_{i \in {S}} \, {B^{+}_i} \leq T} \right)
    \leq k \exp \left( { - {{2{t^2}}}/{m}} \right), 
\end{align}
where $t={(\mu_S-\mu_{S^c})}/{2}$. For the second term, we have
\begin{align}
    \begin{gathered}
  \mathbb{P}\left( {\mathop {\max }\limits_{j \in {C\backslash S}} \,{B^{-}_j} \geq T} \right)
    \leq \sum\limits_{j \notin S} {\mathbb{P}\left( {E_j} \right) \cdot \mathbb{P}\left( { B_j^{-} \geq T \mid E_j} \right)} \hfill \\
    \leq \left( {n - k} \right)\left(1-p(1-p)^k\right)^m \exp \left( { -{{2{t^2}}}/{m}} \right). \hfill \\ 
\end{gathered} 
\end{align}
Thus, the probability of at least one false negative satisfies
\begin{align}\label{bernoulli_false_negative_ineq_k}
    \begin{gathered}
        \mathbb{P}\left( {\text{at least one false negative under Top-k strategy}} \right) \hfill \\
        \leq k \exp \left( { -{{2{t^2}}}/{m}} \right) + \left( {n - k} \right)\exp \left( { -mp(1-p)^k -{{2{t^2}}}/{m}} \right). \hfill \\
    \end{gathered}
\end{align}

Let $\epsilon\in(0,1)$ denote the target failure probability. Under the Top-k strategy, we have $\mathbb{P}\{S \subseteq {{\hat S}_k}\}\geq 1-\epsilon$, provided that
\begin{align}\label{pf_eq:unconstrained_m}
    m \geq \frac{{2\left(1-p(1-p)^k\right)^2\log \left( {2n/\epsilon}\right)}}{{{p^2}{{\left( {1 - p} \right)}^{2k + 2}}}}.
\end{align}
Since $1\geq 1-p(1-p)^k \geq 1-p$, the bound is simplified to
\begin{align}
    m \geq \frac{{2\log \left( {2n/\epsilon}\right)}}{{{p^2}{{\left( {1 - p} \right)}^{2k + 2}}}}.
\end{align}

\paragraph{Top-2k sample selecting strategy} There exists false negatives if at least $k+1$ negative samples $j\in C\backslash S$ have $B^{-}_{j}\geq \mathop {\min }\limits_{i \in {S}} \, B^{+}_{i}$, which can be expressed as
\begin{align}\label{bernoulli_top2k_ineq}
    \begin{gathered}
        \mathbb{P}\left( {\mathop {\min }\limits_{i \in {S}} \, {B^{+}_i} \leq T\,and\,\sum\limits_{j \in {C\backslash S}} {\mathbb{I}\left( {{B^{-}_j} \geq {T}} \right)}  \geq k + 1} \right) \hfill \\
        \leq\mathbb{P}\left( {\mathop {\min }\limits_{i \in {S}} \, {B^{+}_i} \leq T} \right) + \mathbb{P}\left( {\sum\limits_{j \in {C\backslash S}} {\mathbb{I}\left( {{B^{-}_j} \geq {T}} \right)}  \geq k + 1} \right). \hfill \\
    \end{gathered}
\end{align}
Set $T={(\mu_S+\mu_{S^c})}/{2}$ as before. Applying Lemma~\ref{hoeffding}, the second term on the right-hand side of inequality~(\ref{bernoulli_top2k_ineq}) satisfies
\begin{align}
    \begin{gathered}
  \mathbb{P}\left( {\sum\limits_{j \in {C\backslash S}} {\mathbb{I}\left( {{B^{-}_j} \geq {T}} \right)}  \geq k + 1} \right)  \hfill \\
   \leq \binom{n-k}{k+1} \left(\mathbb{P}\left(E_j \,and\, {{B^{-}_j} \geq T} \right)\right)^{k+1}  \hfill \\
   \leq \binom{n-k}{k+1}  \exp \left( { -m(k+1)p(1-p)^k - {{2\left( {k + 1} \right){t^2}}}/{m}} \right), \hfill \\ 
\end{gathered} 
\end{align}
where $t={(\mu_S-\mu_{S^c})}/{2}$. Thus, we have
\begin{align}\label{bernoulli_false_negative_ineq_2k}
    \begin{gathered}
        \mathbb{P}\left({\text{at least one false negative under Top-2k strategy}}\right) \hfill \\ 
        \leq \binom{n-k}{k+1} \exp \left( { -m(k+1)p(1-p)^k - {{2\left( {k + 1} \right){t^2}}}/{m}} \right) \hfill \\
        \quad + k\exp \left( { - {{2{t^2}}}/{m}} \right) \hfill \\ 
    \end{gathered}
\end{align}

Under the Top-2k strategy, using the same number of pooled tests $m$ which  satisfies inequality~(\ref{pf_eq:unconstrained_m}), we have $\mathbb{P}\{S \subseteq {{\hat S}_{2k}}\}\geq 1-\delta$, where
\begin{align}
    \delta  = {\left( {{{e\epsilon }}/{{{(k + 1)}}}} \right)^{k + 1}} + {k\epsilon }/{{n}}.
\end{align}
Hence, the proof of Theorem~\ref{thm:unconstrained} is completed.
\end{proof}

\section{Proof of Theorem~\ref{thm:dilution_constrained}}\label{pf:dilution_constrained}
\begin{proof}[Proof of Theorem \ref{thm:dilution_constrained}.]
In the dilution-constrained design, each row of the pooling matrix $\Phi$ has $P_{r}$ ones uniformly distributed. The probability that a given sample $q\in\{1,\ldots,n\}$ is assigned to a pool $l$ is ${P_{r}}/{n}$, i.e., 
\begin{align}
\mathbb{P}\left(\Phi_{l,q}=1\right)={P_r}/{n}\,,\,q=1,\ldots,n.
\end{align}
A pool is negative if and only if it contains no positive samples, which happens with probability
\begin{align}
    q_1 = {\binom{n-k}{P_r}}/{\binom{n}{P_r}}.
\end{align}
Thus, the probability that a pool is positive is $1-q_1$. For each negative sample $j$, we define $E_j$ as the event of $j\in C$, which occurs with probability $\left(1 - q_1 \cdot \frac{P_r}{n-k}\right)^m$.

Following the analysis in the proof of Theorem~\ref{thm:unconstrained}, for a positive sample $i\in S$, the distribution of $B^{+}_{i}$ is
\begin{align}
    {B^{+}_i} \sim Binomial\left( {m,{P_r}/{n}} \right),
\end{align}
with $\mu_S=\mathbb{E}\left[B^{+}_{i}\right]$. For a negative sample $j\in C$, the conditional distribution of $B^{-}_{j}$ is 
\begin{align}
    {B^{-}_j}\mid E_j \sim Binomial\left( {m,\frac{\frac{P_r}{n}\cdot \left(1-q_2\right)}{1-q_1 \cdot \frac{P_r}{n-k}}} \right),
\end{align}
with $\mu_{S^{c}}=\mathbb{E}\left[B^{-}_{j}\right]$, 
where $q_2=\frac{\binom{n-k-1}{P_r-1}}{\binom{n-1}{P_r-1}}$, satisfying $q_2=\frac{n}{n-k}q_1$.

The following proof is analogous to that of Theorem~\ref{thm:unconstrained}. Set $T={(\mu_S+\mu_{S^{c}})}/{2}$ and $t = {({{\mu _S} - \mu_{S^{c}}})}/{2}$. Substituting them into inequalities~(\ref{bernoulli_false_negative_ineq_k}) and~(\ref{bernoulli_false_negative_ineq_2k}), we can obtain the results stated in Theorem~\ref{thm:dilution_constrained}.
\end{proof}

\section{Proof of Theorem~\ref{thm:sample_constrained}}\label{pf:sample_constrained}
\begin{proof}[Proof of Theorem \ref{thm:sample_constrained}.]
In the sample-constrained design, for each column $j$ of $\Phi$, we make $P_c$ independent selections from the set $\{1,\dots,m\}$ and set $\Phi_{i,j}=1$ if index $i$ is selected at least once. Thus, each sample is assigned to at most $P_c$ pools. 

\paragraph{Top-k sample selecting strategy}
There exists false negatives if at least one negative sample $j\in C\backslash S$ has $B^{-}_j=P_c$. The probability of this event satisfies
\begin{align}
    \mathbb{P}\left( {\sum\limits_{j \in {C\backslash S}} {\mathbb{I}\left( {{B^{-}_j} = P_c} \right)}  \geq 1} \right) 
    \leq \sum\limits_{j \notin S} {\mathbb{P}\left( {{E_j}\,and\,B_j^ - = {P_c}} \right)}, 
\end{align}
where $E_j$ is the event that the negative sample $j$ is in the candidate set $C$.

For a negative sample $j\in C\backslash S$, $B^{-}_{j}=P_c$ if and only if the sample is assigned to $P_c$ pools and all of these pools test positive. Thus,
\begin{align}
    \begin{gathered}
        \mathbb{P}\left(E_j \,and\,B^{-}_{j}=P_c\right) \hfill \\
        =\mathbb{P}\left(\left|\mathrm{supp} \left(\Phi_{:,j}\right)\right|=P_c\right) \hfill \\
        \quad \cdot \mathbb{P}\left(\mathrm{supp}\left(\Phi_{:,j}\right)  \subseteq \mathrm{supp}\left( \tilde{y} \right) \mid\left|\mathrm{supp} \left(\Phi_{:,j}\right)\right|=P_c\right), \hfill \\
    \end{gathered} 
\end{align}
where $\Phi_{:,j}$ is the column $j$ of $\Phi$, $\tilde{y}$ is the binary measurement vector, and $\mathrm{supp}\left(\cdot\right)$ represents the support set of a vector, defined as the set of indices corresponding to its nonzero entries. Since each column is generated by making $P_c$ independent selections from the set $\{1,\ldots,m\}$, the probability that all selected indices are distinct is
\begin{align}\label{eq:col_eq_1}
    \mathbb{P}\left(\left|\mathrm{supp} \left(\Phi_{:,j}\right)\right|=P_c\right)=\frac{{m\left( {m - 1} \right) \cdots \left( {m - P_c + 1} \right)}}{{{m^{P_c}}}}.
\end{align}
Conditioned on a specific support set $\mathrm{supp}\left(\tilde{y}\right)$ and the event $\left|\mathrm{supp} \left(\Phi_{:,j}\right)\right|=P_c$, the probability that all selected pools are positive is given by
\begin{align}
    \begin{gathered}
        \mathbb{P}\left(\mathrm{supp}\left(\Phi_{:,j}\right)  \subseteq \mathrm{supp}\left( \tilde{y} \right) \mid \left|\mathrm{supp} \left(\Phi_{:,j}\right)\right| =P_c, \mathrm{supp}\left(\tilde{y}\right)\right) \hfill \\
        = {\binom{\left|\mathrm{supp}\left(\tilde{y}\right)\right|}{P_c}}/{\binom{m}{P_c}}. \hfill \\
    \end{gathered}
\end{align}
Since there are at most $kP_c$ positive pools, we can have
\begin{align}\label{ineq:col_ineq_2}
    \begin{gathered}
    \mathbb{P}\left(\mathrm{supp}\left(\Phi_{:,j}\right)  \subseteq \mathrm{supp}\left( \tilde{y} \right) \mid \left|\mathrm{supp} \left(\Phi_{:,j}\right)\right|=P_c\right) \hfill \\
    \leq {\binom{kP_c}{P_c}}/{\binom{m}{P_c}}. \hfill \\
    \end{gathered}
\end{align}
Combining~(\ref{eq:col_eq_1}) and~(\ref{ineq:col_ineq_2}), yields
\begin{align}
        \mathbb{P}\left(E_j \,and\,B^{-}_j=P_c\right) \leq {\left( {{{kP_c}}/{m}} \right)^{P_c}}.
\end{align}
Thus, the probability of at least one false negative satisfies
\begin{align}
    \begin{gathered}
        \mathbb{P}\left( {\text{at least one false negative under Top-k strategy}} \right) \hfill \\
        \leq \left( {n - k} \right) \left( {{{kP_c}}/{m}} \right)^{P_c}. \hfill \\
    \end{gathered}
\end{align}

Let $\epsilon\in\left(0,1\right)$ denote the target failure probability. Under the Top-k strategy, we have $\mathbb{P}\{S \subseteq {{\hat S}_k}\}\geq 1-\epsilon$, provided that
\begin{align}\label{pf_eq:sample-constrained_m}
    m \geq kP_c{\left( {{({n - k})}/{\epsilon}} \right)^{\frac{1}{P_c}}}.
\end{align}

\paragraph{Top-2k sample selecting strategy} 
There exists false negatives if at least $k+1$ negative samples in the candidate set have $B^{-}_j=P_c$. The probability of this event satisfies
\begin{align}
    \begin{gathered}
        \mathbb{P}\left( {\text{at least one false negative under Top-2k strategy}} \right) \hfill \\
        \leq \binom{n-k}{k+1}\left(\mathbb{P}\left(E_j\,and\,B^{-}_j=P_c\right)\right)^{k+1} \hfill \\
        \leq {\left( {{{e\left( {n - k} \right)}}/{{(k + 1)}}} \right)^{k + 1}}{\left( {{{kP_c}}/{m}} \right)^{\left( {k + 1} \right)P_c}}. \hfill \\
    \end{gathered}
\end{align}
Under the Top-2k strategy, using the same number of pooled tests $m$ satisfying inequality~(\ref{pf_eq:sample-constrained_m}), we have $\mathbb{P}\{S \subseteq {{\hat S}_{2k}}\}\geq 1-\delta$, where
\begin{align}
    \delta  = {\left( {{{e\epsilon }}/{({k + 1})}} \right)^{k + 1}}.
\end{align}
This completes the proof of Theorem~\ref{thm:sample_constrained}.
\end{proof}

\section*{References}
\vspace{-0.1\baselineskip}
\begingroup
\makeatletter
\def\section*#1{}%
\bibliographystyle{IEEEtran}
\bibliography{IEEEabrv,reference}
\makeatother
\endgroup




\end{document}


\onecolumn
\renewcommand{\thefigure}{S\arabic{figure}}
\renewcommand{\thetable}{S\arabic{table}}
\renewcommand{\theequation}{S\arabic{equation}}
\renewcommand{\thepage}{S\arabic{page}}
\setcounter{figure}{0}
\setcounter{table}{0}
\setcounter{equation}{0}
\setcounter{page}{1} 

\section*{Supplementary Material}
\subsection{Supplementary Tables}
\begin{table}[H]
\caption{Experimentally optimal parameters for pooling designs in LoSc ($n=10000$)}\label{tab:combinatorial_optimal pooling design_n10000}
\centering
\begin{tabular}{ccccccccccccc}
\hline
& \multicolumn{4}{c}{\textbf{Unconstrained}}                                                             & \multicolumn{4}{c}{\textbf{Dilution-Constrained}}                                                            & \multicolumn{4}{c}{\textbf{Sample-Constrained}}                                                         \\ \cline{2-5} \cline{6-9} \cline{10-13} 
                        & \multicolumn{2}{c}{\textbf{Top-k}}                   & \multicolumn{2}{c}{\textbf{Top-2k}}                  & \multicolumn{2}{c}{\textbf{Top-k}}                   & \multicolumn{2}{c}{\textbf{Top-2k}}                  & \multicolumn{2}{c}{\textbf{Top-k}}                   & \multicolumn{2}{c}{\textbf{Top-2k}}                  \\ \cline{2-3} \cline{4-5} \cline{6-7} \cline{8-9} \cline{10-11} \cline{12-13} 
                        \boldmath{$k$}& \boldmath{$p^*$} & \boldmath{$m^*$} & \boldmath{$p^*$} & \boldmath{$m^*$} & \boldmath{$P^*_r$} & \boldmath{$m^*$} & \boldmath{$P^*_r$} & \boldmath{$m^*$} & \boldmath{$P^*_c$} & \boldmath{$m^*$} & \boldmath{$P^*_c$} & \boldmath{$m^*$} \\ \hline
2                & 0.2500                     & \textbf{60}                      & 0.2500                     & \underline{50}                     & 2500                     & 70                     & 2500                    & 60                     &  9                    & 60                     &  6                    & 60                     \\
10                 &  0.0625                    & 350                     & 0.0625                     & 220                    & 312                     & 440                     &  625                    &  220                    &  13
                    & \textbf{300}                     &  9                    & \underline{180}                     \\
20                 &  0.0313
                    & 710                     & 0.0313
                     &  400                    &  312
                    & 700                     & 312
                     &  390                    &  19
                    & \textbf{600}                     & 9
                     &  \underline{300}                    \\
30                 & 0.0156
                     & 1160                     & 0.0313
                     & 520                     & 312
                     & 980                     & 312
                     & 510                     & 13                     & \textbf{840}                     & 6
                     & \underline{420}                     \\
40                 & 0.0156
                     & 1470                     & 0.0156
                     & 690                     &  156
                    &  1380                    & 156
                     & 680                    & 19
                     & \textbf{1080}                     &   6
                   & \underline{540}                     \\
50                 & 0.0156
                     & 1890                     & 0.0156
                     & 800                    & 156
                     & 1770                     & 156
                     & 800                    & 27
                     & \textbf{1440}                    & 6
                     & \underline{660}                    \\
100                   &  0.0078
                    &  3540                    & 0.0078
                     &  1350                    & 78
                     &  3840                    & 78
                     &  1350                    & 19
                     & \textbf{2820}                     & 6                     & \underline{1080}                     \\
200                   & 0.0039                    & 7520                     & 0.0039                     & 2290                     & 39                     & 7210                     & 39                     & 2230                     & 19                     & \textbf{5400}                     & 6                     & \underline{1740}                     \\
300                   & -                    & -                     & 0.0039                     & 3070                     & -                     & -                     & 39                     & 3060                     & 19                     & \textbf{8160}                     & 4                     & \underline{2340}                     \\
400                   & -                     & -                     & 0.0020                     & 3710                     & -                     & -                     & 19                    & 3620                     & -                     & -                     & 4                     & \underline{2820}                     \\
500                  & -                     & -                     & 0.0020                     & 4180                     & -                     & -                     & 19                     & 4110                     & -                     & -                     & 4                     & \underline{3240}                     \\
1000                  & -                     & -                     & 0.0010                     & 6190                     & -                     & -                     & 9                     & 5900                     & -                     & -                     & 3                     & \underline{4800}                     \\ \hline
\end{tabular}

\vspace{3pt}
\begin{minipage}{\textwidth}
\small\justifying
\noindent Each design is evaluated at its experimentally optimal pooling parameter ($p^*$, $P^*_r$ or $P^*_c$). Here, $m^*$ denotes the minimum number of pooled tests required for a $99\%$ success rate. Optimal values of $m^*$ under the Top-k strategy are marked in bold, and optimal values of $m^*$ under the Top-2k strategy are underlined. Some results are not reported as the number of pooled tests required exceeds the population size.
\end{minipage}
\end{table}

\begin{table}[!ht]
\caption{Comparison between theoretical and experimental optimal parameters under the Top-k strategy ($n=10000$)}\label{tab:combinatorial_optimal pooling design_n10000_compare}
\centering
\begin{tabular}{ccccccccccccc}
\hline
& \multicolumn{4}{c}{\textbf{Unconstrained}}                                                             & \multicolumn{4}{c}{\textbf{Dilution-Constrained}}                                                            & \multicolumn{4}{c}{\textbf{Sample-Constrained}}                                                         \\ \cline{2-5} \cline{6-9} \cline{10-13} 

&\multicolumn{2}{c}{\textbf{Experiment}} & \multicolumn{2}{c}{\textbf{Theory}} &\multicolumn{2}{c}{\textbf{Experiment}} & \multicolumn{2}{c}{\textbf{Theory}} &\multicolumn{2}{c}{\textbf{Experiment}} & \multicolumn{2}{c}{\textbf{Theory}} \\ \cline{2-3} \cline{4-5} \cline{6-7} \cline{8-9} \cline{10-11} \cline{12-13}

                        \boldmath{$k$}& \boldmath{$p^*$} & \boldmath{$m^*$} & \boldmath{$p^*$} & \boldmath{$m^*$} & \boldmath{$P_r^*$} & \boldmath{$m^*$} & \boldmath{$P_r^*$} & \boldmath{$m^*$} & \boldmath{$P_c^*$} & \boldmath{$m^*$} & \boldmath{$P_c^*$} & \boldmath{$m^*$} \\ \hline
2                & 0.2500                     & 60                      & 0.3333                     & 87                      & 2500                     & 70                     & 3333                    & 87                     &  9                    & 60                     &  14                    & 76                     \\
10                 &  0.0625                    & 350                     & 0.0909                     & 388                    & 312                     & 440                     &  909                    &  387                    &  13
                    & 300                     &  14                    & 376                     \\
20                 &  0.0313
                    & 710                     & 0.0476
                     &  763                    &  312
                    & 700                     & 476
                     &  762                    &  19
                    & 600                     & 14
                     & 752                    \\
30                 & 0.0156
                     & 1160                     & 0.0323
                     & 1139                     & 312
                     & 980                     & 322
                     & 1137                     & 13                     & 840                     & 14
                     & 1127                     \\
40                 & 0.0156
                     & 1470                     & 0.0244
                     & 1514                     & 156
                    &  1380                    & 243
                     & 1511                    & 19
                     & 1080                     &   14                   & 1502                     \\
50                 & 0.0156
                     & 1890                     & 0.0196
                     & 1889                    & 156
                     & 1770                     & 196
                     & 1884                    & 27
                     & 1440                    & 14
                     & 1878                    \\
100                   &  0.0078
                    &  3540                    & 0.0099
                     &  3765                    & 78
                     &  3840                    & 99
                     &  3746                    & 19
                     & 2820                     & 14                     & 3754 \\ \hline
\end{tabular}

\vspace{3pt}
\begin{minipage}{\textwidth}
\small\justifying
\noindent Theoretical results for the unconstrained and dilution-constrained designs are obtained from the corresponding conjectures, whereas those for the sample-constrained design are obtained from the theorem.
\end{minipage}
\end{table}

\newpage
\twocolumn
\subsection{Supplementary Figures}
\begin{figure}[H]
    \centering
    \captionsetup{position=below,justification=centering,singlelinecheck=false}
    \subcaptionbox{Unconstrained design}[1\linewidth]{\includegraphics[width=1\linewidth]{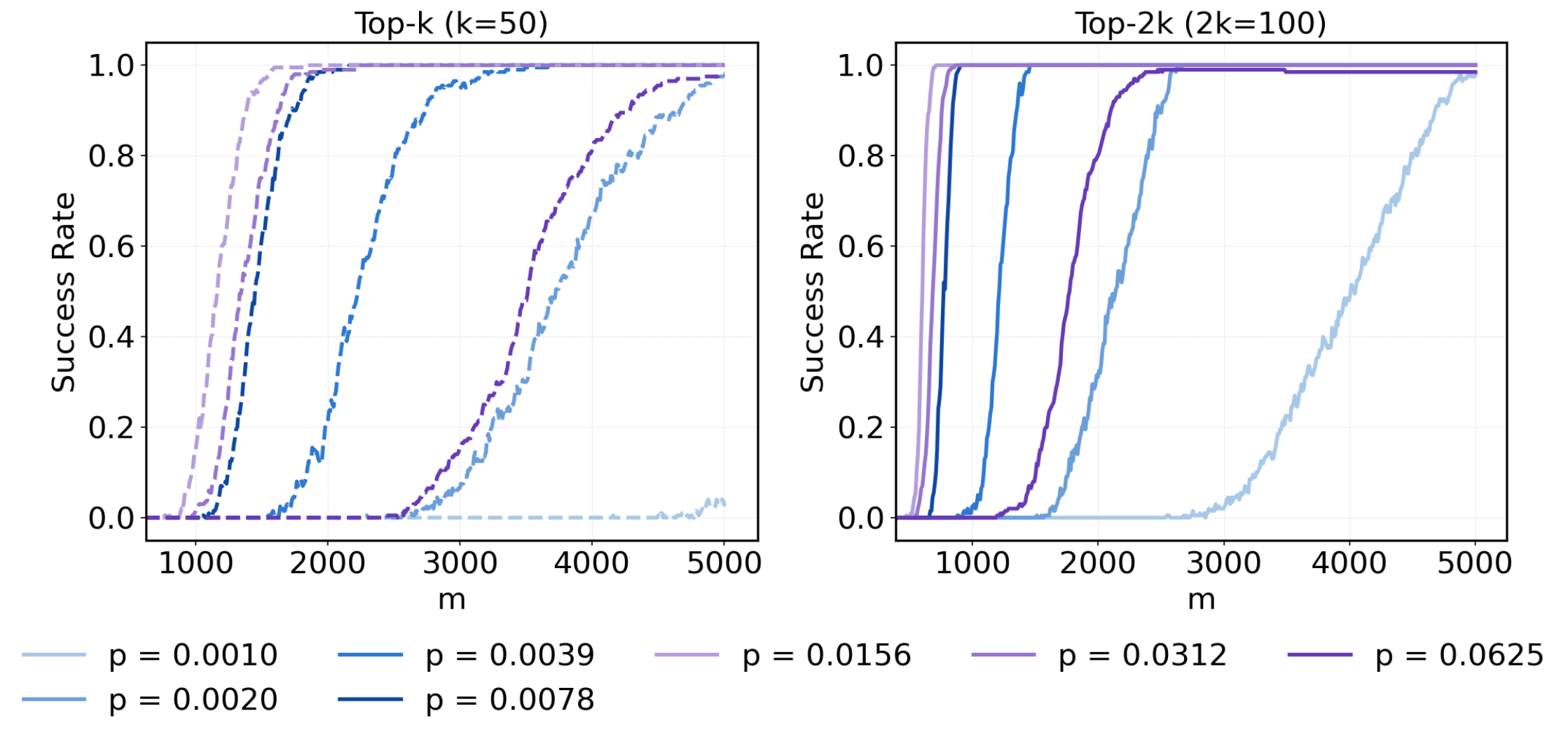}}
    \quad
    \subcaptionbox{Dilution-constrained design}[1\linewidth]{\includegraphics[width=1\linewidth]{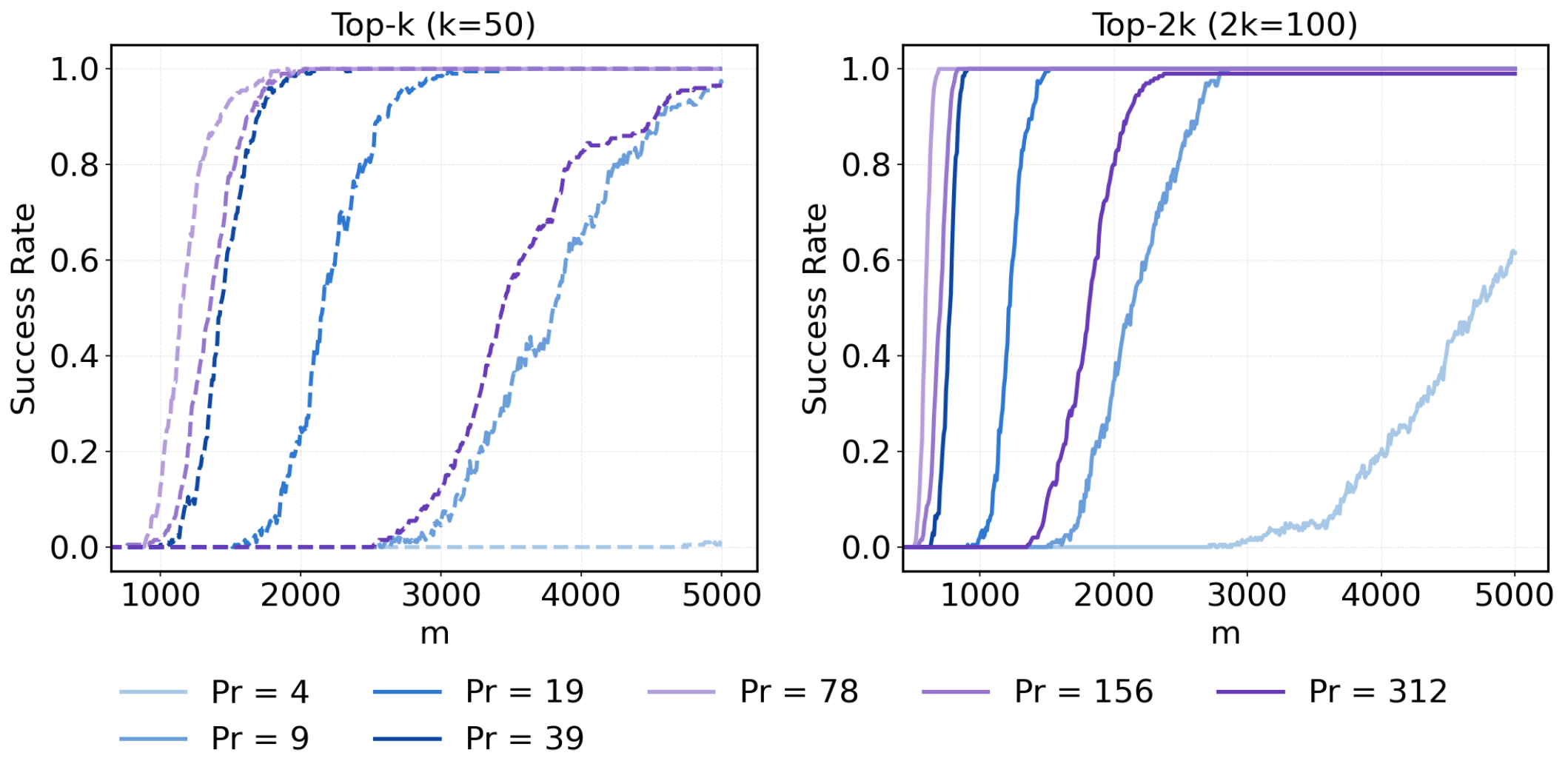}}
    \quad
    \subcaptionbox{Sample-constrained design}[1\linewidth]{\includegraphics[width=1\linewidth]{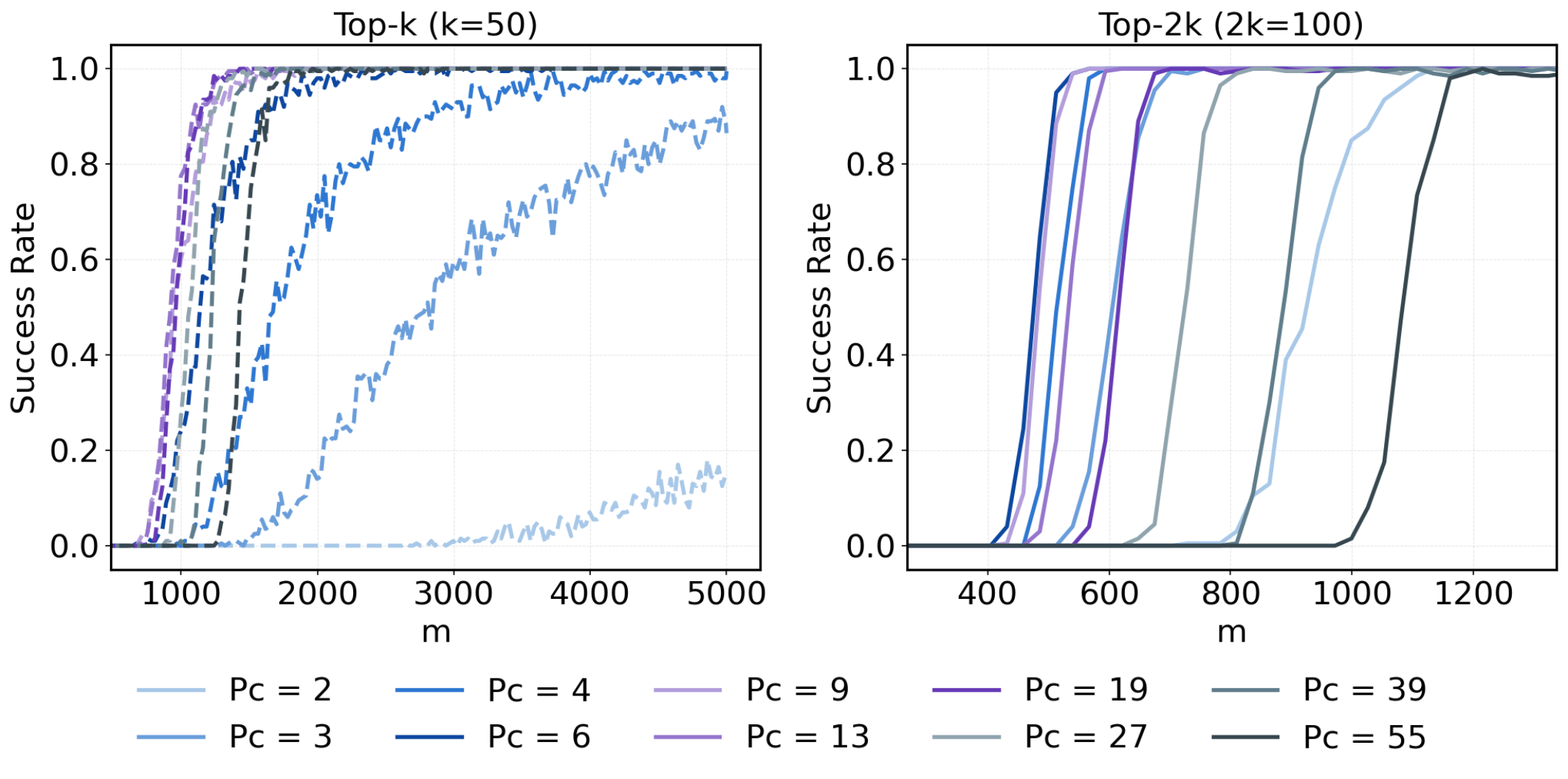}}
    \captionsetup{position=below,justification=justified,singlelinecheck=false}
    \caption{\textbf{Impact of pooling parameters on recovery performance.} Recovery success rate as a function of $m$ under different pooling parameters when $n=5000$ and $k=50$.}
    \label{fig:combinatorial_success_rate_n5000_kr0.01}
\end{figure}

\newpage
\subsection*{}
\begin{figure}[H]
    \centering
    \captionsetup{position=below,justification=centering,singlelinecheck=false}
    \subcaptionbox{Unconstrained design}[1\linewidth]{\includegraphics[width=1\linewidth]{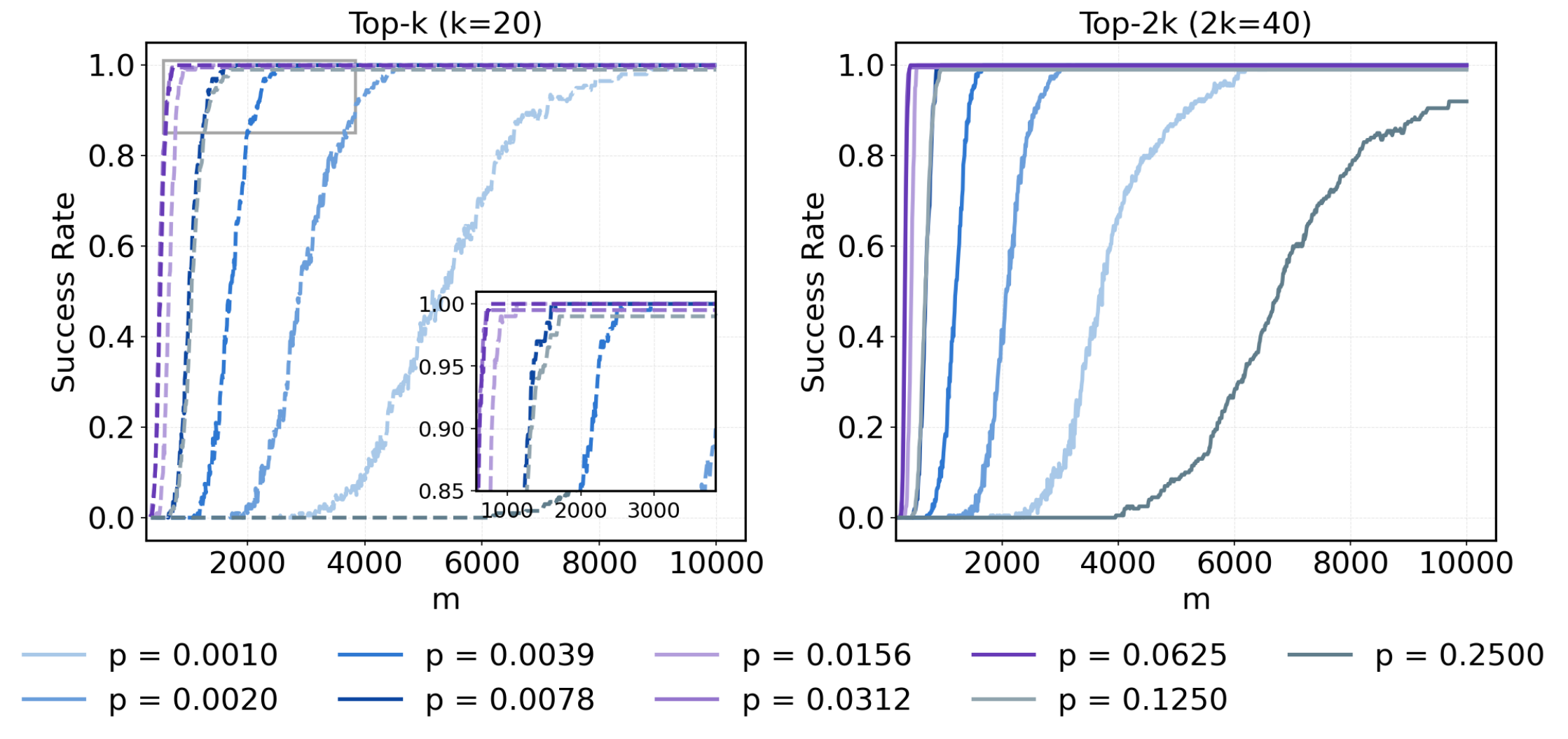}}
    \quad
    \subcaptionbox{Dilution-constrained design}[1\linewidth]{\includegraphics[width=1\linewidth]{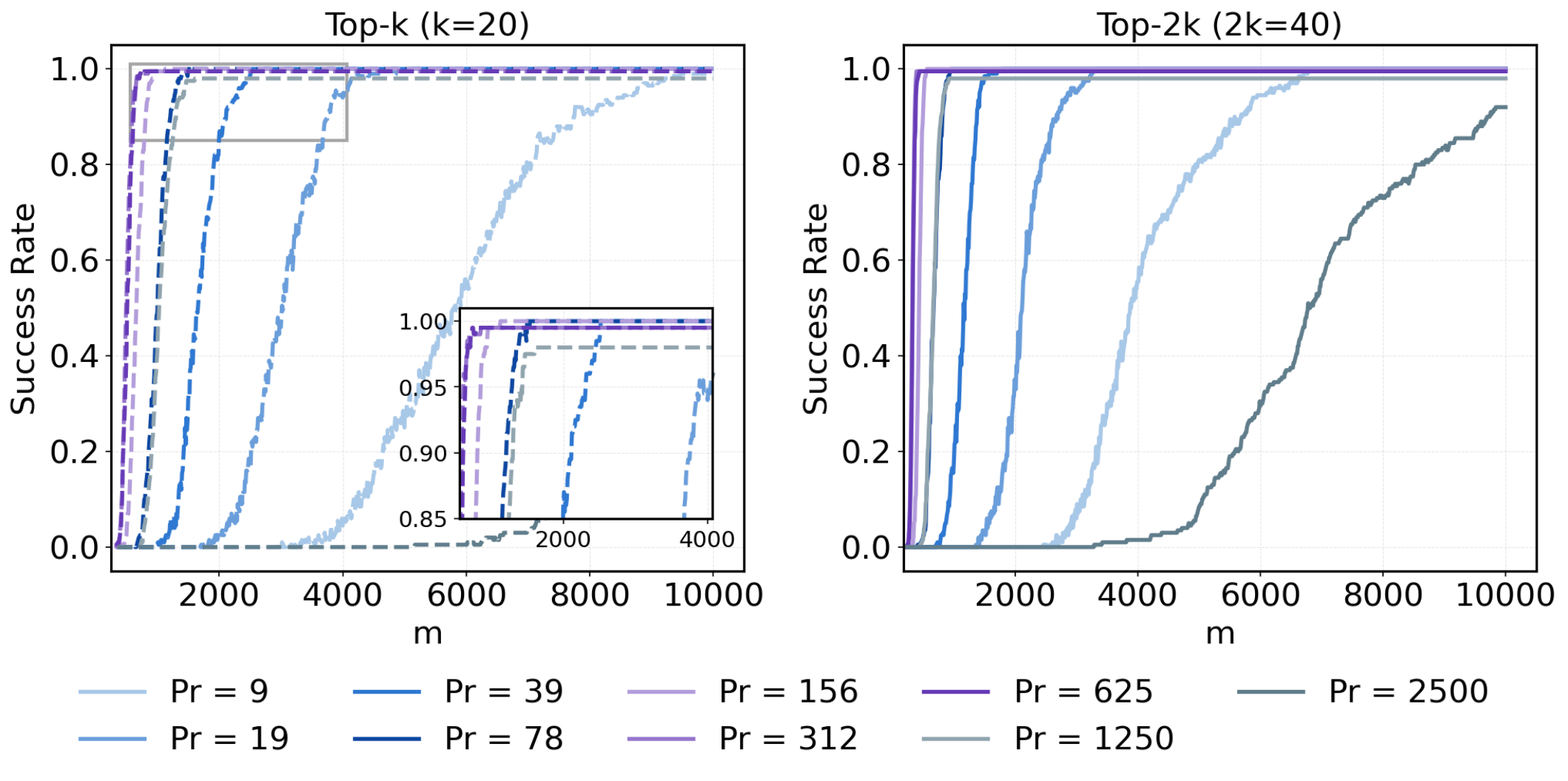}}
    \quad
    \subcaptionbox{Sample-constrained design}[1\linewidth]{\includegraphics[width=1\linewidth]{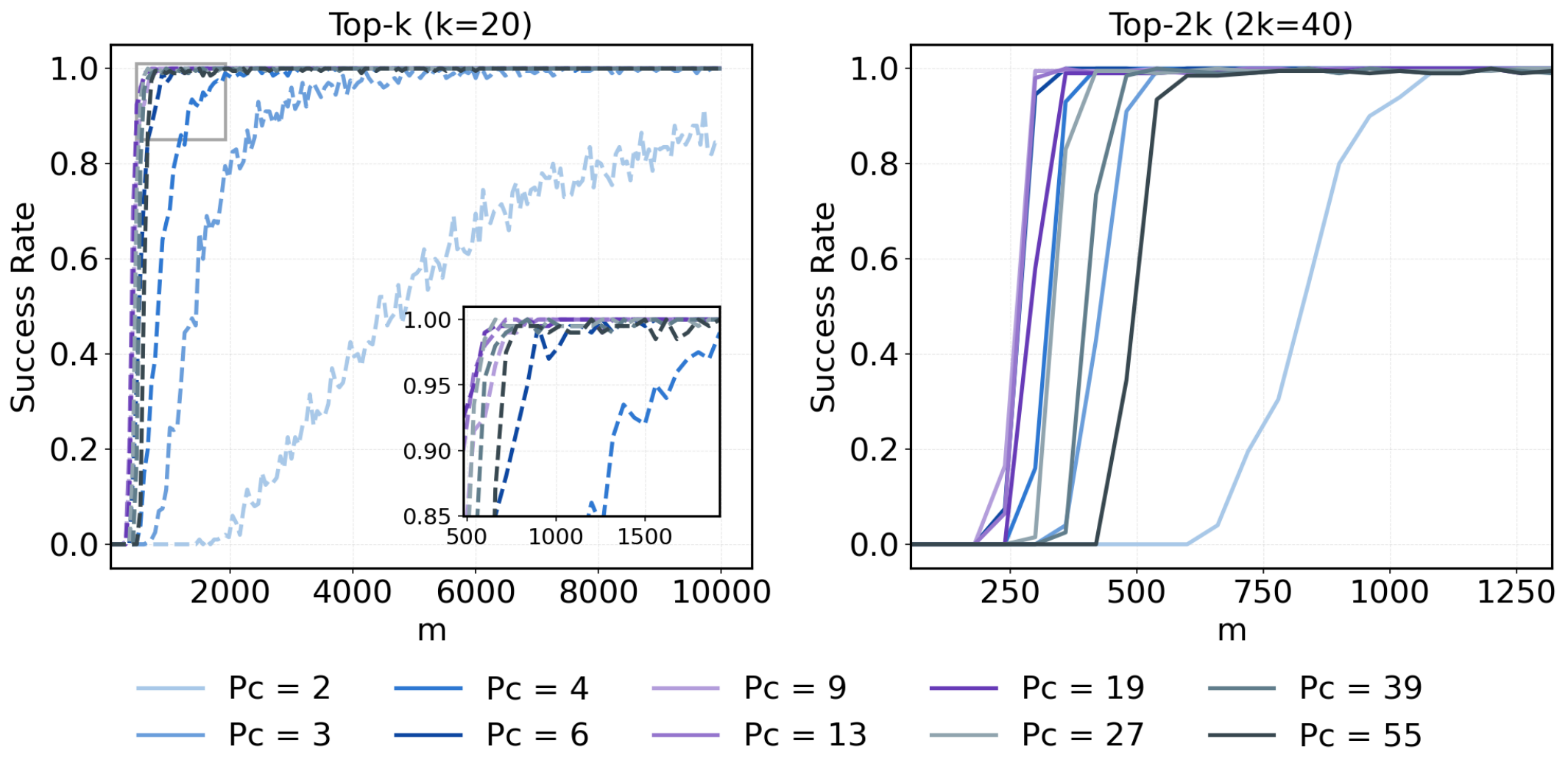}}
    \captionsetup{position=below,justification=justified,singlelinecheck=false}
    \caption{\textbf{Impact of pooling parameters on recovery performance.} Recovery success rate as a function of $m$ under different pooling parameters when $n=10000$ and $k=20$.}
    \label{fig:combinatorial_success_rate_n10000_kr0.002}
\end{figure}

\newpage
\begin{figure}[H]
    \centering
    \captionsetup{position=below,justification=centering,singlelinecheck=false}
    \subcaptionbox{Unconstrained design}[1\linewidth]{\includegraphics[width=1\linewidth]{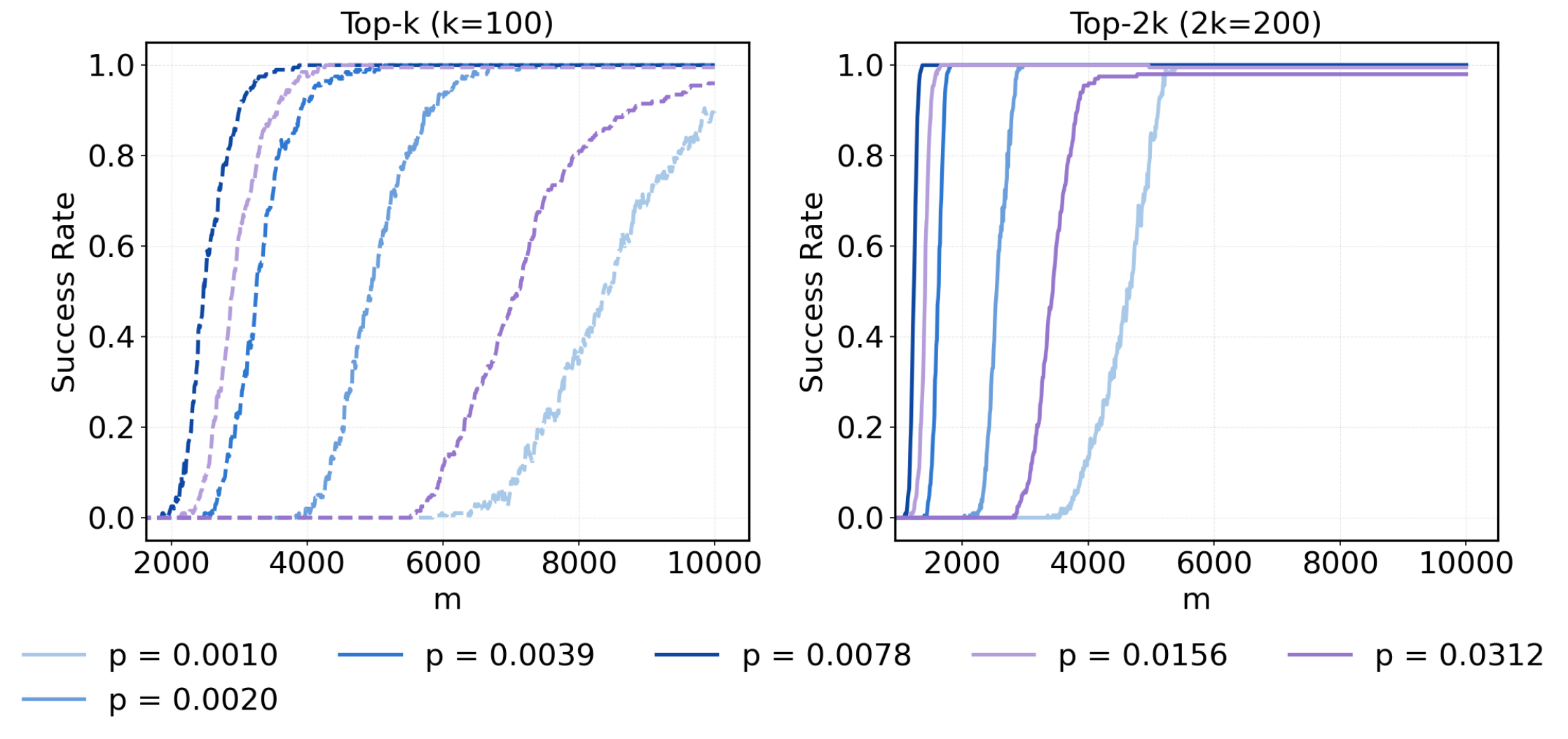}}
    \quad
    \subcaptionbox{Dilution-constrained design}[1\linewidth]{\includegraphics[width=1\linewidth]{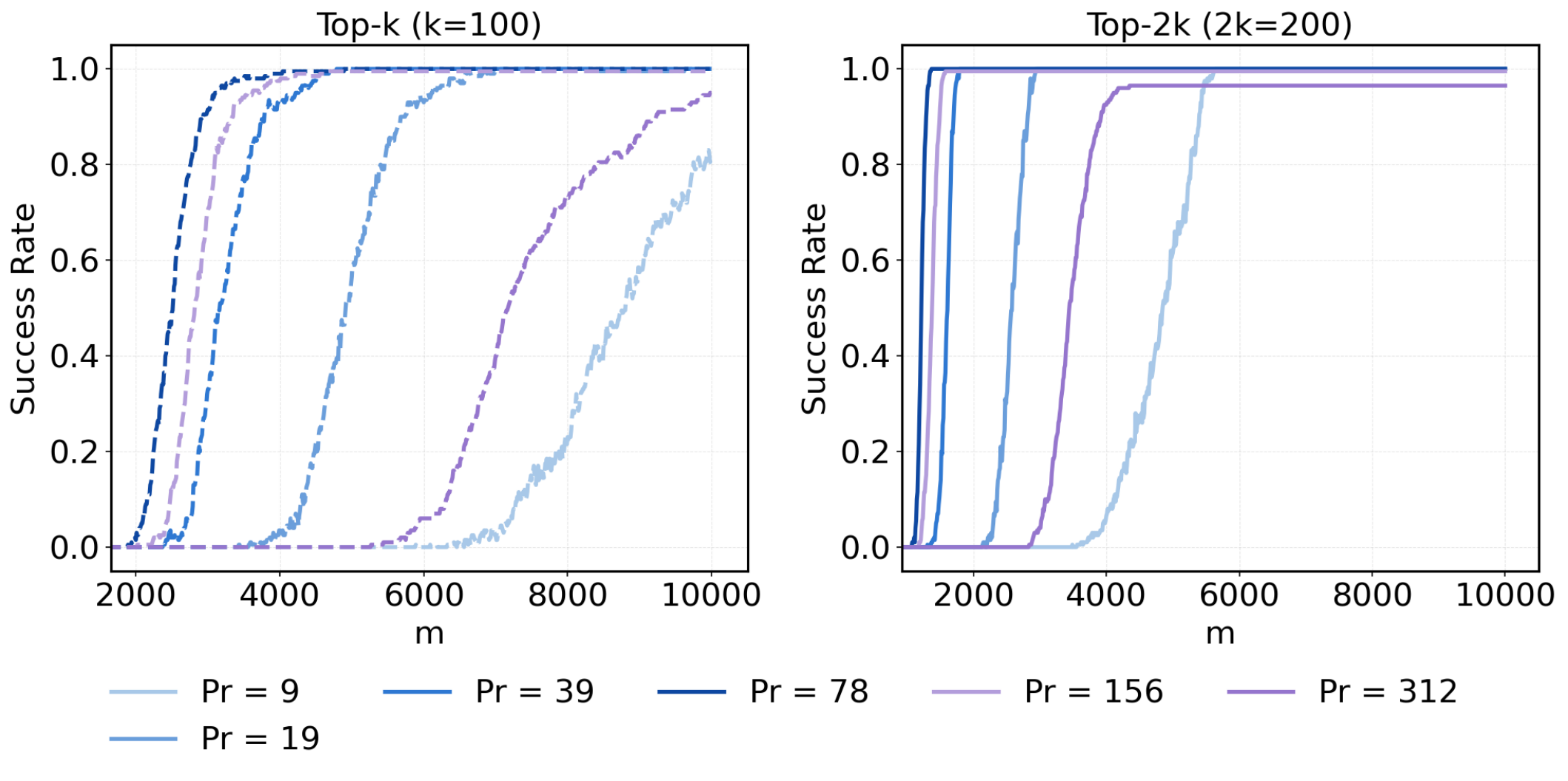}}
    \quad
    \subcaptionbox{Sample-constrained design}[1\linewidth]{\includegraphics[width=1\linewidth]{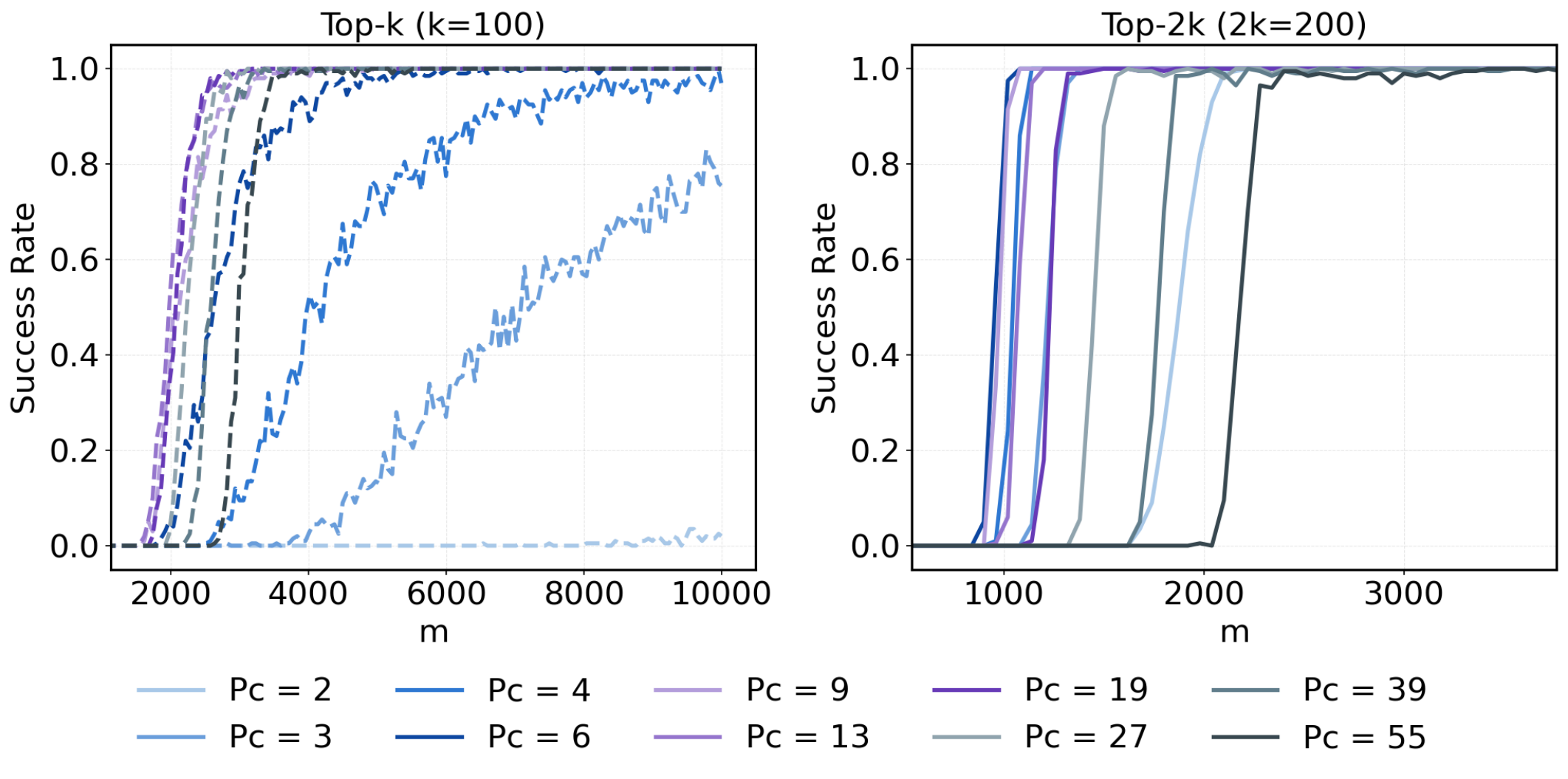}}
    \captionsetup{position=below,justification=justified,singlelinecheck=false}
    \caption{\textbf{Impact of pooling parameters on recovery performance.} Recovery success rate as a function of $m$ under different pooling parameters when $n=10000$ and $k=100$.}
    \label{fig:combinatorial_success_rate_n10000_kr0.01}
\end{figure}

\onecolumn
\begin{figure}[!t]
\centering
\includegraphics[width=0.88\textwidth]{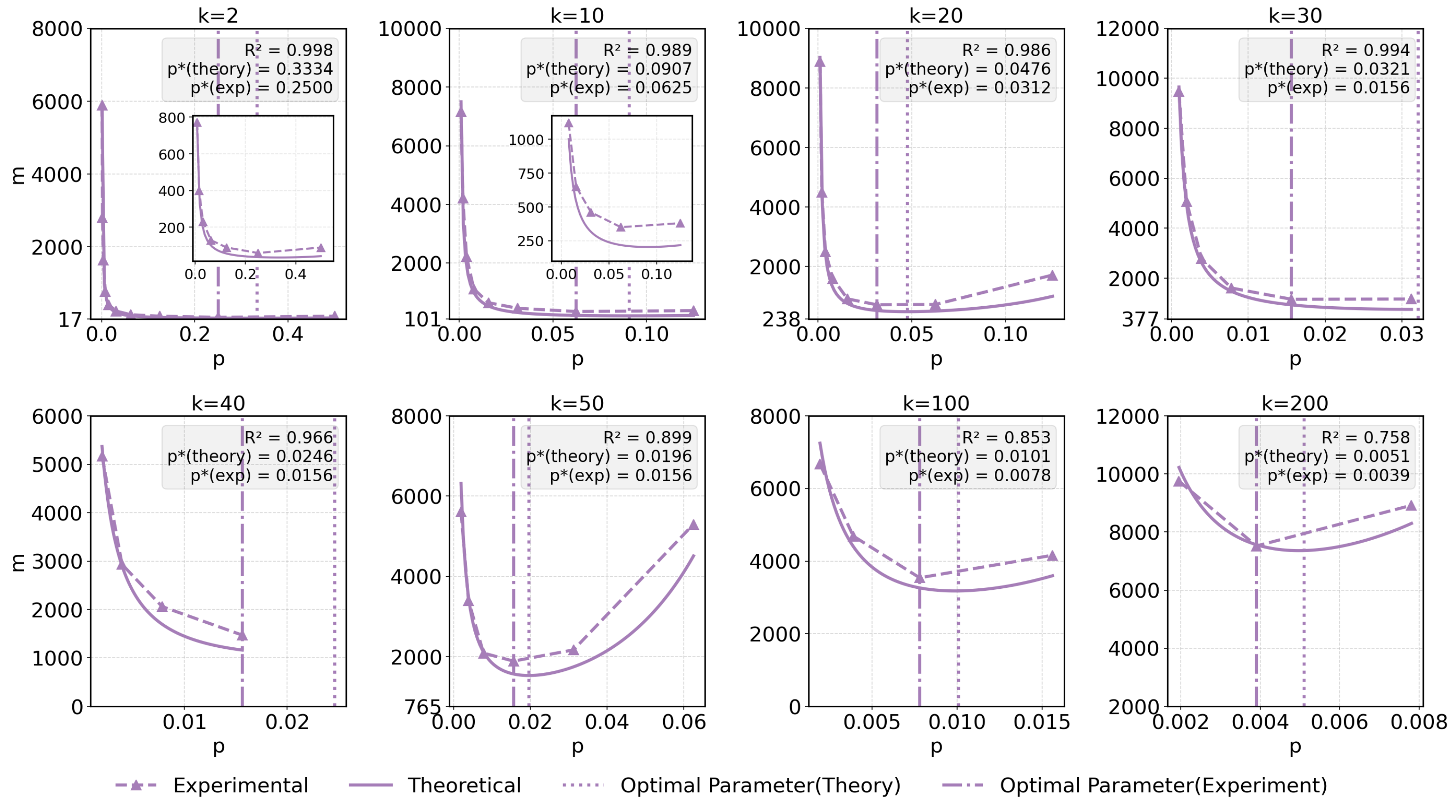}
\caption{\textbf{Experimental and theoretical results of \boldmath{$m^*(p)$} for unconstrained design.} $m^*(p)$ is the function of $p$ under the Top-k strategy when $n=10000$. Vertical lines mark the theoretically and experimentally optimal pooling parameters.}
\label{fig:combinatorial_fit_n10000_bernoulli_k}
\end{figure}

\begin{figure}[!t]
\centering
\includegraphics[width=0.88\textwidth]{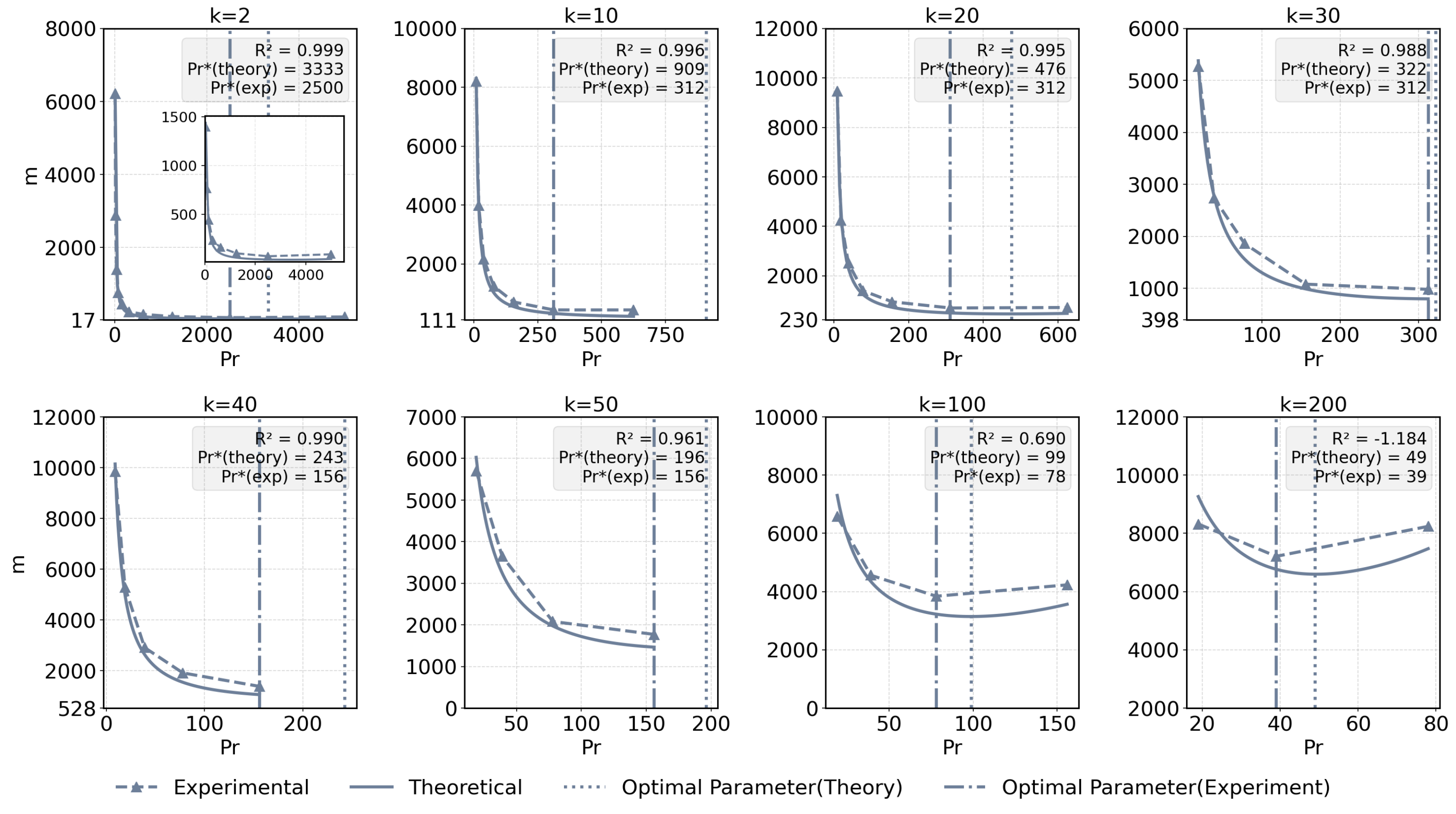}
\caption{\textbf{Experimental and theoretical results of \boldmath{$m^*\left(P_r\right)$} for dilution-constrained design.} $m^*\left(P_r\right)$ is the function of $P_r$ under the Top-k strategy when $n=10000$. Vertical lines mark the theoretically and experimentally optimal pooling parameters.}
\label{fig:combinatorial_fit_n10000_row_k}
\end{figure}

\begin{figure}[H]
\centering
\includegraphics[width=0.88\textwidth]{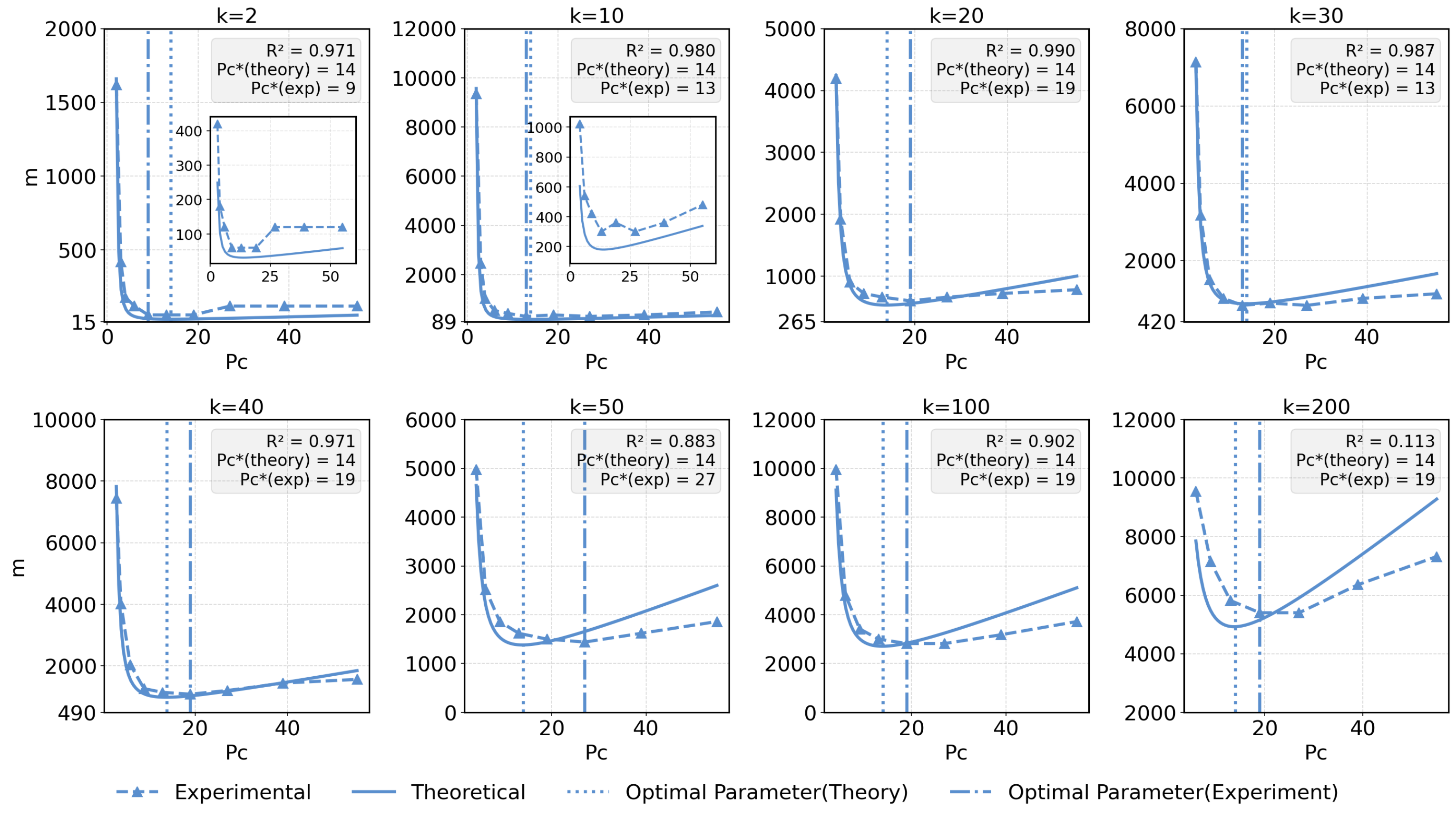}
\caption{\textbf{Experimental and theoretical results of \boldmath{$m^*\left(P_c\right)$} for sample-constrained design.} $m^*\left(P_c\right)$ is the function of $P_c$ under the Top-k strategy when $n=10000$. Vertical lines mark the theoretically and experimentally optimal pooling parameter.}
\label{fig:combinatorial_fit_n10000_col_k}
\end{figure}